\documentclass[10pt,twocolumn,letterpaper]{article}

\usepackage[pagenumbers]{cvpr} 
\usepackage[accsupp]{axessibility}

\usepackage{booktabs}
\usepackage{multirow}

\usepackage{ragged2e}

\usepackage{amssymb}
\usepackage{makecell}
\usepackage{multirow}
\usepackage{rotating}
\usepackage{array}
\usepackage{times}
\usepackage{graphicx}
\usepackage{psfrag}
\usepackage{enumitem}
\usepackage{url}
\usepackage{amsmath}
\usepackage{booktabs}
\usepackage{lscape}
\usepackage{verbatim}
\usepackage{overpic}
\usepackage{bbding}

\usepackage{amsthm}

\usepackage{tablefootnote}

\usepackage{array}
\usepackage{makecell}

\usepackage[ruled,linesnumbered]{algorithm2e}
\SetKwComment{Comment}{/* }{ */}

\newtheorem{theorem}{Theorem}

\newtheorem{remark}{Remark}

\usepackage[misc]{ifsym}

\definecolor{cvprblue}{rgb}{0.21,0.49,0.74}
\usepackage[pagebackref,breaklinks,colorlinks,citecolor=cvprblue]{hyperref}

\title{Interpreting Object-level Foundation Models via Visual Precision Search\vspace{-16pt}}

\author{
Ruoyu Chen$^{1,2}$, Siyuan Liang$^{3}$, Jingzhi Li$^{1,2,7}$, Shiming Liu$^{4}$, Maosen Li$^{5}$, \\
Zhen Huang$^{6}$, Hua Zhang$^{1,2,*}$, and Xiaochun Cao$^{8,}$\thanks{Corresponding authors.}\\
\normalsize$^{1}$Institute of Information Engineering, Chinese Academy of Sciences, Beijing 100093, China\\
\normalsize$^{2}$School of Cyber Security, University of Chinese Academy of Sciences, Beijing 100049, China\\
\normalsize$^{3}$School of Computing, NUS~~~$^{4}$RAMS Lab, Huawei Technologies Co., Ltd.~~~$^{5}$IAS BU, Huawei Technologies Co., Ltd.\\
\normalsize$^{6}$College of Computer, NUDT~$^{7}$Key Lab. of Edu. Inf. for Nationalities (YNNU), Ministry of Education, Kunming, China\\
\normalsize$^{8}$School of Cyber Science and Technology, Shenzhen Campus of Sun Yat-sen University, Shenzhen 518107, China\\
\small\texttt{chenruoyu@iie.ac.cn}~~~~~~~~~~~~\texttt{pandaliang521@gmail.com}~~~~~~~~~~~~\texttt{\{lijingzhi,zhanghua\}@iie.ac.cn}\\
\small\texttt{\{liushiming3,limaosen2\}@huawei.com}~~~~~~\texttt{huangzhen@nudt.edu.cn}~~~~~~\texttt{caoxiaochun@mail.sysu.edu.cn}\vspace{-24pt}
}

\begin{document}
\maketitle

\begin{abstract}
Advances in multimodal pre-training have propelled object-level foundation models, such as Grounding DINO and Florence-2, in tasks like visual grounding and object detection. However, interpreting these models’ decisions has grown increasingly challenging. Existing interpretable attribution methods for object-level task interpretation have notable limitations: (1) gradient-based methods lack precise localization due to visual-textual fusion in foundation models, and (2) perturbation-based methods produce noisy saliency maps, limiting fine-grained interpretability. To address these, we propose a Visual Precision Search method that generates accurate attribution maps with fewer regions. Our method bypasses internal model parameters to overcome attribution issues from multimodal fusion, dividing inputs into sparse sub-regions and using consistency and collaboration scores to accurately identify critical decision-making regions. We also conducted a theoretical analysis of the boundary guarantees and scope of applicability of our method. Experiments on RefCOCO, MS COCO, and LVIS show our approach enhances object-level task interpretability over SOTA for Grounding DINO and Florence-2 across various evaluation metrics, with faithfulness gains of 23.7\%, 31.6\%, and 20.1\% on MS COCO, LVIS, and RefCOCO for Grounding DINO, and 50.7\% and 66.9\% on MS COCO and RefCOCO for Florence-2. Additionally, our method can interpret failures in visual grounding and object detection tasks, surpassing existing methods across multiple evaluation metrics. The code is released at \url{https://github.com/RuoyuChen10/VPS}.
\end{abstract}

\vspace{-18pt}
\begin{figure}[!t]
    \centering
    \includegraphics[width=0.45\textwidth]{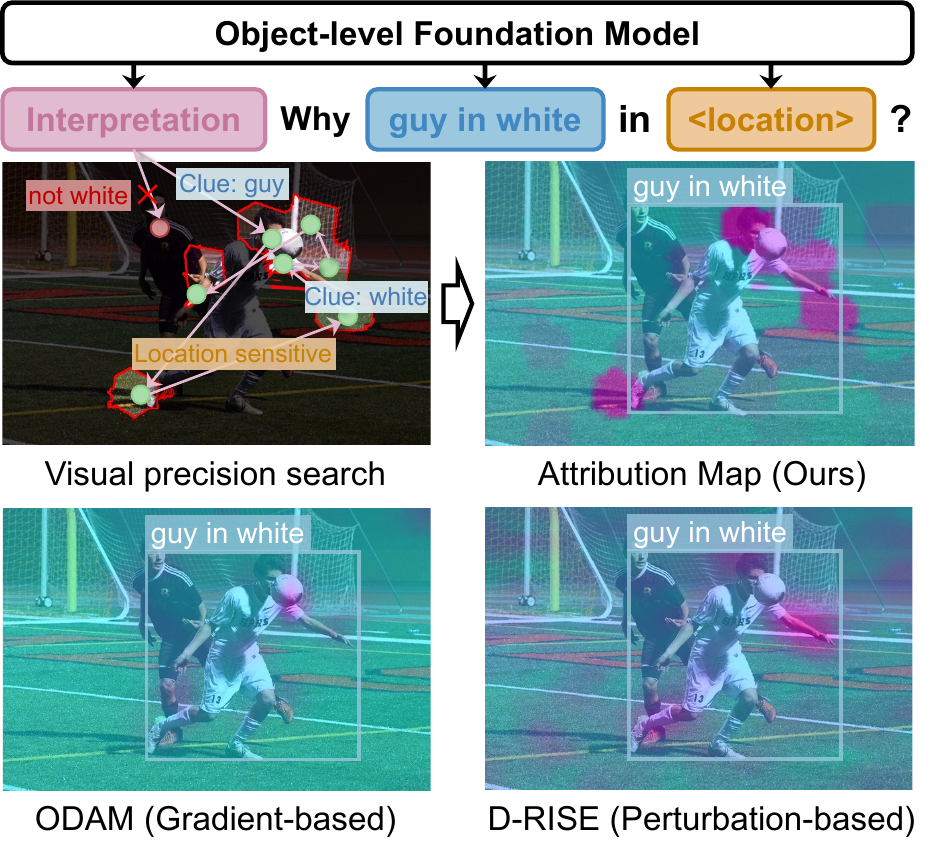}\vspace{-10pt} 
    \caption{Illustration of our Visual Precision Search interpretation method, which more precisely identifies key sub-regions in object-level foundation model decision-making compared to gradient-based and perturbation-based methods.
    }  
    \label{motivation}\vspace{-10pt}
\end{figure}

\section{Introduction}\label{sec:intro}
Understanding object information in images, such as object detection~\cite{he2018mask,cheng2024yolo,liang2024object,zou2023object,zang2024contextual,yao2024detclipv3}, is a crucial and enduring challenge in computer vision, holding significant importance across various fields, including autonomous driving~\cite{wen2024road,chen2024end,hu2023planning}. With advancements in multimodal alignment technology~\cite{radford2021learning,li2022grounded}, object-level foundation models~\cite{wu2024general,liu2023grounding,xiao2024florence,zou2023generalized} like Grounding DINO~\cite{liu2023grounding} and Florence-2~\cite{xiao2024florence} have been developed to handle various tasks, including visual grounding and object detection. However, the massive data volume, rich semantic concepts, and large parameters significantly reduce the transparency~\cite{liang2021parallel,liang2022large,liu2023x} and interpretability~\cite{gandelsman2024interpreting,zhao2024gradient,chen2023sim2word} of these models. Since object-level tasks like autonomous driving demand high model reliability~\cite{feng2021review,wilson2023safe,miller2019evaluating}, building transparent and interpretable models is crucial for enhancing safety and trustworthy~\cite{chen2024less,jiang2024comparing,liang2024object,wei2019transferable,liang2021parallel,liang2022imitated}, making it essential to interpret object-level foundation models.

Some interpretable attribution algorithms for traditional object detection models have been proposed, including gradient-based methods like ODAM~\cite{zhao2024gradient_detector} and perturbation-based methods like D-RISE~\cite{petsiuk2021black}. However, as shown in the lower panel of Figure~\ref{motivation}, these methods face limitations when applied to multimodal foundation models: (1) gradient-based methods may struggle to provide accurate visual localization explanations due to the fusion of visual and textual features in foundation models, and (2) perturbation-based approaches can introduce sampling artifacts, resulting in noisy saliency maps with limited fine-grained interpretability.

To address these issues, we propose a novel interpretation mechanism for the object-level foundation model, as shown in Figure~\ref{motivation}. Our goal is to generate a saliency map that explains the rationale behind the model’s detection of a specific object, allowing fewer regions to be exposed for accurate detection. Additionally, removing these critical regions should quickly lead to detection failure, highlighting their importance in the model’s decision. Specifically, we propose Visual Precision Search, which sparsifies the input region into a series of sub-regions using super-pixel segmentation, and then ranks these sparse sub-regions to determine their importance for object-level decision-making. Our gradient-free method \textit{circumvents localization errors at the visual level that are caused by gradient back-propagation in vision-text fusion}. Furthermore, we propose a novel submodular function, grounded in theoretical analysis and offering boundary guarantees for Visual Precision Search. This function identifies regions that enhance interpretability from two key perspectives: clues that support the model’s accurate detection and regions with strong combinatorial effects. By employing region search to iteratively expand the set of sub-regions, our method more precisely identifies key sub-regions in the decision-making of object-level foundation models, which \textit{mitigates the issue of coarse explanations caused by noise in perturbation-based methods}. In addition, our method effectively analyzes cases of grounding or detection failure, allowing us to observe which input-level factors influence the decision. 

We validated our method on the MS COCO~\cite{lin2014microsoft}, RefCOCO~\cite{kazemzadeh2014referitgame}, and LVIS~\cite{gupta2019lvis} datasets. The object-level foundation models evaluated include Grounding DINO, which uses a multimodal feature fusion architecture, and Florence-2, which employs a multimodal large language model architecture. The tasks covered include interpreting visual grounding and object detection. In explaining the data behind correct model decisions, our method demonstrates that the faithfulness metric for Grounding DINO surpasses SOTA method D-RISE by 23.7\%, 20.1\%, and 31.6\% on MS COCO, RefCOCO, and LVIS, respectively. For Florence-2, our method improves the faithfulness metric by 102.9\% and 66.9\% on MS COCO and RefCOCO, respectively, while also achieving SOTA performance on location metrics across all datasets. In interpreting the factors leading to Grounding DINO’s failures in visual grounding tasks, our method improves the Insertion metric by 42.9\% and the average highest score by 25.1\%. It also identifies causes of detector misclassification and undetection, achieving Insertion metric improvements of 54.7\% and 42.7\% for misclassification, and 36.7\% and 64.3\% for undetection on MS COCO and LVIS, respectively.

In summary, the contributions of this paper are:
\begin{itemize}
    \item We introduce Visual Precision Search, a new search-based mechanism for interpreting object-level foundation models through instance-specific saliency maps.
    \item A novel submodular mechanism is constructed to enhance interpretability from two aspects: identifying clues that support accurate detection and highlighting regions with strong combinatorial effects, with an analysis of its theoretical boundaries for object-level tasks.
    \item We demonstrate our approach’s generalizability across multimodal foundation models, including both non-LLM-based (Grounding DINO) and LLM-based detectors (Florence-2).
    \item We validate our method on MS COCO, RefCOCO, and LVIS, achieving significant improvements in explaining object-level tasks. Additionally, we analyze grounding and detection failures, establishing a quantitative benchmark for this purpose.
\end{itemize}
\section{Related Work}
\label{sec:related}

\textbf{Object Detection.} 
Object detection involves locating objects in images and identifying their categories~\cite{liang2024object}. Early models typically used convolutional neural network (CNN) backbones, divided into two main types: two-stage and single-stage models. Two-stage methods, like Faster R-CNN~\cite{ren2016faster} and Mask R-CNN~\cite{he2018mask}, generate region proposals to identify potential foreground objects before refining localization and classification. In contrast, single-stage methods, such as YOLO~\cite{wang2023yolov7} and FCOS~\cite{tian2020fcos}, directly classify and regress on backbone features. Transformer architectures~\cite{vaswani2017attention} have further advanced object detection, as exemplified by DETR~\cite{carion2020end}, which uses a fully end-to-end Transformer-based approach. The rise of multimodal foundation models~\cite{radford2021learning,li2022grounded} has enabled open-set detection tasks~\cite{liang2024object}, allowing object grounding with textual input~\cite{wu2024towards,cheng2024yolo}. GLIP~\cite{li2022grounded} frames detection as a grounding problem, integrating vision and language through feature fusion in the neck module. Grounding DINO~\cite{liu2023grounding} uses multi-phase vision-language fusion, achieving state-of-the-art (SOTA) referring expression comprehension (REC)~\cite{liu2017referring}. Other methods explore multimodal large language models (MLLMs) for detection and grounding tasks. Florence-2~\cite{xiao2024florence}, for instance, introduces a prompt-based model that generates object categories and coordinates using text prompts directly. The diverse feature extractors, architectures, and outputs in object detection make it challenging to interpret their decisions using a single unified approach.

\textbf{Object Detection Explanation.} 
Explaining object detector decisions remains a largely unexplored area. Gudovskiy \textit{et al.}~\cite{gudovskiy2018explain} used Integrated Gradients (IG)~\cite{sundararajan2017axiomatic} and SHAP~\cite{scott2017unified} with SSD~\cite{liu2016ssd} classification scores for box-level attribution. Lee \textit{et al.}~\cite{lee2021bbam} optimized masks on Mask R-CNN~\cite{he2018mask} to approximate original decisions with minimal pixels. Petsiuk \textit{et al.}~\cite{petsiuk2021black} proposed D-RISE, an adaptation of RISE~\cite{petsiuk2018rise}, though it often introduces noise in saliency maps. Jiang \textit{et al.}~\cite{jiang2023diverse} introduced Nesterov-Accelerated iGOS++, which can be resource-intensive for large models. Zhao \textit{et al.}~\cite{zhao2024gradient_detector} developed ODAM using Grad-CAM~\cite{selvaraju2020grad}, but its effectiveness depends on network layer choices. SSGrad-CAM++~\cite{yamauchi2024spatial} generates saliency maps based on Grad-CAM++~\cite{chattopadhay2018grad}. VX-CODE~\cite{yamauchi2024explaining} attributes important patches based on greedy search and SHAP~\cite{scott2017unified}. While these methods aim to interpret detector errors, they largely rely on simple visualizations and empirical observations, often overlooking input-level feature confusion that can obscure human understanding. Masking confusing regions may aid in correcting model outputs~\cite{chen2024less}. Our paper introduces a black-box interpretation approach for object-level foundation models using Visual Precision Search, enhancing detection performance with fewer regions and identifying input-level failure causes to improve detection accuracy.

\section{Method}
\label{sec:method} 

This section presents our method for explaining the decisions of object-level foundation models. 
Figure~\ref{framework_figure} illustrates the overall framework of our method.

\begin{figure*}
    \vspace{-30pt}
    \centering
    \includegraphics[width=\textwidth]{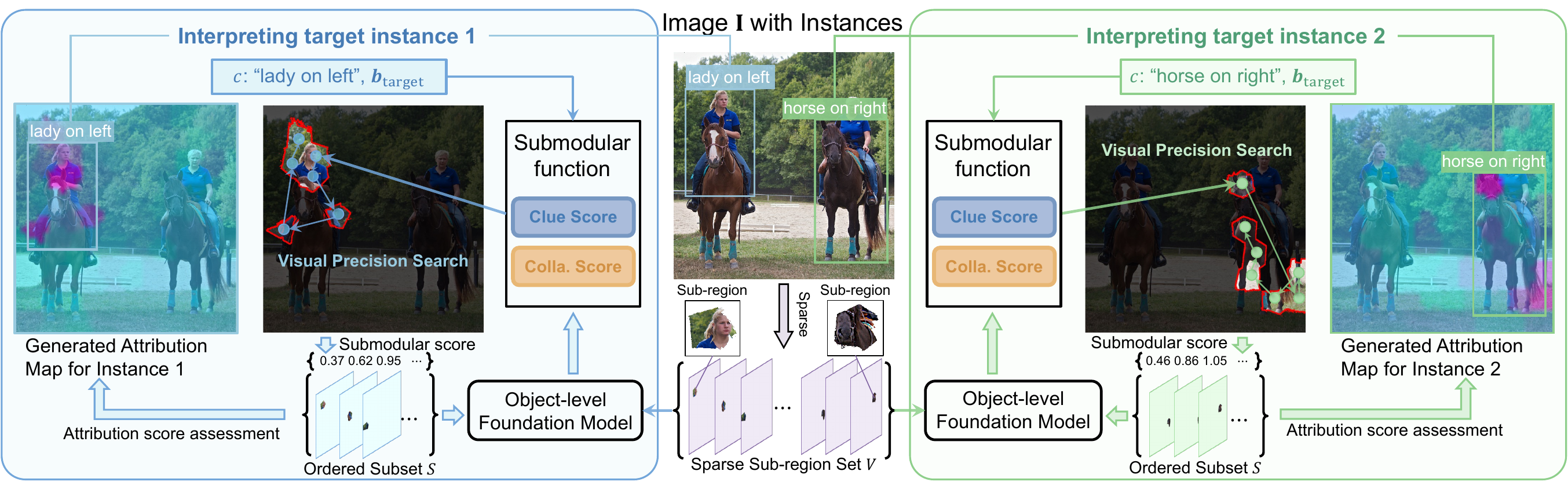}\vspace{-10pt} 
    \caption{Framework of the proposed Visual Precision Search method for interpreting an object-level foundation model. The input is first sparsified into a set of sub-regions and then interpreted across different instances. A submodular function guides the search for significant sub-regions, updating the ordered subset iteratively, and ultimately generating the instance-level attribution map.
    } 
    \label{framework_figure}
    \vspace{-15pt}
\end{figure*}

\subsection{Problem Formulation}\label{problem_formulation}

Given an image $\mathbf{I} \in \mathbb{R}^{h \times w \times 3}$ and an object-level foundation model $f(\cdot)$, the output can be represented as $f(\mathbf{I}) = \{(\boldsymbol{b}_i, c_i, s_{i}) \mid i = 1, 2, \ldots, N\}$. Each tuple $(\boldsymbol{b}_i, c_i, s_i)$ denotes the bounding box, class label, and confidence score for each detected object $i$, where $N$ is the maximum number of detection boxes. Tasks like object detection involve multiple boxes and categories, while visual grounding tasks like referring expression comprehension (REC) have only one box and category.
Our goal is to generate a saliency map that explains the reasons behind the model’s detection of a specific object, $(\boldsymbol{b}_i, c_i, s_i)$. Ideally, this explanation would enable the correct detection of objects using fewer input regions, while also requiring the removal of only a few critical regions to invalidate the detector. To achieve this, we can sparsify the input region $V = \{\mathbf{I}^{s}_1, \ldots, \mathbf{I}^{s}_m \}$, where $\mathbf{I}^{s}_i$ represents the $i$-th sub-region, $m$ is the total number of sub-regions, and $\mathbf{I} = \sum_{i=1}^{m} \mathbf{I}^{s}_i$. We then frame the saliency map generation task as a subset selection problem for these input sub-regions. A set function $\mathcal{F}(\cdot)$ is defined to assess interpretability by determining whether a given region is a key factor in the model’s decision. Therefore, the objectives are:
\begin{equation}\label{problem_objective}
    \max_{S \subseteq V, |S|<k}\mathcal{F}(S),
\end{equation}
where $k$ denotes the maximum number of sub-regions. Therefore, the key to this problem lies in designing set function $\mathcal{F}$ and optimizing Eq.~\ref{problem_objective}.

\subsection{Visual Precision Search}\label{proposed_method}

We propose a Visual Precision Search method for interpreting object-level models. To begin, the input region needs to be sparsified. We apply the SLICO superpixel segmentation algorithm~\cite{achanta2012slic} to divide the input into $m$ sub-regions, $V = \{\mathbf{I}^{s}_1, \ldots, \mathbf{I}^{s}_m \}$. To solve Eq.~\ref{problem_objective}, an $\mathcal{NP}$-hard problem, we employ submodular optimization~\cite{fujishige2005submodular}. Since the saliency map requires all sub-regions to be ranked, $k$ can be set to $|V|$ to compute ordered subsets. When the set function $\mathcal{F}(\cdot)$ satisfies the properties of diminishing returns\footnote{Diminishing returns: $\forall S_{A} \subseteq S_{B} \subseteq V \setminus \alpha, \; \mathcal{F}(S_{A} \cup \{\alpha\}) - \mathcal{F}(S_{A}) \ge \mathcal{F}(S_{B} \cup \{\alpha\}) - \mathcal{F}(S_{B})$.} and monotonic non-negative\footnote{Monotonic non-negative: $\forall S \subseteq V \setminus \alpha, \; \mathcal{F}(S \cup \{\alpha\}) - \mathcal{F}(S) \ge 0$.}, a greedy search guarantees an approximate optimal solution~\cite{edmonds1970submodular}, with $\mathcal{F}(S) \ge (1 - 1/e)\mathcal{F}(S^{\ast})$, where $S^{\ast}$ denotes the optimal solution and $S$ denotes the greedy solution. We next design a set function to evaluate the interpretability score and rank the importance of interpretable regions for object-level tasks.

\textbf{Clue Score:} An essential aspect of interpretability is enabling the object-level foundation model to accurately locate and identify objects while using fewer regions. To assess the importance of subregions from this perspective, we introduce the object location box information, $\boldsymbol{b}_{\text{target}}$, and the target category, $c$, that requires explanation. Given a subregion $S$, the object-level model outputs $N$ instances, denoted as $f(S) = \{(\boldsymbol{b}_i, s_{c,i}) \mid i=1, \ldots, N\}$, where $s_{c,i}$ represents the confidence of the $i$-th bounding box in predicting category $c$. Then, the clue score of sub-region $S$ is defined as:
\begin{equation}\label{clue_score}
    s_{\text{clue}}(S,\boldsymbol{b}_{\text{target}},c) = \max_{(\boldsymbol{b}_i, s_{c,i})\in f(S)}{\text{IoU}(\boldsymbol{b}_{\text{target}},\boldsymbol{b}_i)\cdot s_{c,i}},
\end{equation}
where $\text{IoU}(\cdot,\cdot)$ represents the Intersection over Union between two bounding boxes. Unlike Petsiuk \textit{et al.}~\cite{petsiuk2021black}, which only considers high-confidence bounding boxes, our method includes all candidate boxes to avoid getting stuck in a local optimum during the search. By incorporating the clue score $s_{\text{clue}}$, our method focuses on regions that strengthen the desired positional and semantic response.

\textbf{Collaboration Score:} Some regions may exhibit strong combination effects, meaning they contribute effectively to model decisions only when paired with multiple specific sub-regions. Therefore, we introduce the collaboration score $s_{\text{colla.}}$ to assess sub-regions with high sensitivity to decision outcomes:
\begin{equation}\label{colla_score}
    s_{\text{colla.}}(S,\boldsymbol{b}_{\text{target}},c) = 1 - \max_{(\boldsymbol{b}_i, s_{c,i})\in f(V \setminus S)}{\text{IoU}(\boldsymbol{b}_{\text{target}},\boldsymbol{b}_i)\cdot s_{c,i}},
\end{equation}
this means that no detection box can accurately or confidently locate the object once the key sub-region is removed. This score serves as an effective guide in the initial search phase to identify subtle key regions that contribute to the final decision.

\textbf{Submodular Function:} The scores above are combined to construct a submodular function $\mathcal{F}(\cdot)$, as follows:
\begin{equation}\label{submodular_function}
    \mathcal{F}(S,\boldsymbol{b}_{\text{target}},c) = s_{\text{clue}}(S,\boldsymbol{b}_{\text{target}},c) + s_{\text{colla.}}(S,\boldsymbol{b}_{\text{target}},c).
\end{equation}

\textbf{Saliency Map Generation:} Using the above submodular function, a greedy search algorithm is applied to sort all sub-regions in $V$, yielding an ordered subset $S$. Introducing the submodular function enables the search algorithm to more precisely identify key visual regions for interpretation. Additionally, scoring the sub-regions is necessary to better explain the importance of each sub-region. We evaluate the salient difference between the two sub-regions by the marginal effect. The attribution score $\mathcal{A}_i$ for each sub-region $s_i$ in $S$ is assessed by:
\begin{equation}
    \mathcal{A}_i = \begin{cases}
     b_{\text{base}} & \text{if } i = 1, \\
     \mathcal{A}_{i-1} - \left| \mathcal{F}(S_{[i]}) - \mathcal{F}(S_{[i-1]}) \right| & \text{if } i>1,
    \end{cases}
\end{equation}
where $b_{\text{base}}$ represents a baseline attribution score (such as 0 and won't affect results) for the first sub-region, and $S_{[i]}$ denotes the set containing the top $i$ sub-regions in $S$. When a new sub-region is added, a small marginal increase suggests comparable importance to the previous sub-region. A negative marginal effect indicates a counterproductive impact, which can be assessed by its absolute value. Finally, $\mathcal{A}$ is normalized to obtain the saliency map of the sub-region. The detailed calculation process of the proposed Visual Precision Search algorithm is outlined in Algorithm~\ref{alg:vps}.

\begin{algorithm}[]
    \footnotesize 
    \caption{Visual Precision Search algorithm for generating interpretable saliency maps of object-level foundation models.}
    \label{alg:vps}
    \KwIn{Image $\mathbf{I} \in \mathbb{R}^{h \times w \times 3}$, a division algorithm $\texttt{Div}(\cdot)$, the object location box $\boldsymbol{b}_{\text{target}}$, and the target category $c$.}
    \KwOut{An ordered set $S$ and a saliency map $\mathcal{A} \in \mathbb{R}^{h \times w}$.}
    $V \gets \texttt{Div}(\mathbf{I})$\;
    $S \gets \varnothing$ \Comment*[r]{Initiate the operation of submodular subset selection}
    $\mathcal{A}_1 \gets b_{\text{base}}$\;
    \For{$i=1$ \KwTo $|V|$}{
        $S_d \gets V \backslash S$\;
        $\alpha \gets  \arg{\max_{\alpha \in S_d}{\mathcal{F} \left( S \cup \{ \alpha \},\boldsymbol{b}_{\text{target}},c \right) }}$\;
        $S \gets S \cup \{ \alpha \}$\;
        \uIf {$i>1$}{
            $\mathcal{A}_i \gets \mathcal{A}_{i-1} - \left| \mathcal{F}(S_{[i]}) - \mathcal{F}(S_{[i-1]}) \right|$
        }
    }
    \Return $S,\text{norm}(\mathcal{A})$
\end{algorithm}

\subsection{Theory Analysis}\label{theory_analysis}

A theoretical analysis is conducted for our Visual Precision Search method. 

\begin{theorem}[Submodular Properties]\label{submodular_properties}
Consider two sub-sets $S_{A}$ and $S_{B}$ in set $V$, where $S_{A} \subseteq S_{B} \subseteq V$. Given an element $\alpha$, where $\alpha \in V \setminus S_{b}$. Assuming that $\alpha$ is contributing to model interpretation,
then, the function $\mathcal{F}(\cdot)$ in Eq.~\ref{submodular_function} satisfies diminishing returns,
$$
\mathcal{F}(S_{A} \cup \{\alpha\}) - \mathcal{F}(S_{A}) \ge \mathcal{F}(S_{B} \cup \{\alpha\}) - \mathcal{F}(S_{B}),
$$
and monotonic non-negative, $\mathcal{F}(S_A \cup \{\alpha\}) - \mathcal{F}(S) \ge 0$, and thus, $\mathcal{F}$ is a submodular function.
\end{theorem}

\begin{proof}
Please see supplementary material for the proof.
\end{proof}

\begin{remark}[Impact on Sparse Division]
The quality of the search space is determined by sparse division, meaning that both the method of partitioning the input and the number of sub-regions play a crucial role in the faithfulness of the Visual Precision Search.
\end{remark}

\begin{remark}[Applicable Models]
If the model provides bounding box positions and category confidence without filtering out low-confidence detection boxes, the interpretation results of the Visual Precision Search can achieve guaranteed optimal boundaries. Models with detection heads are generally applicable; however, MLLM-based models may not directly output confidence scores, leaving room for improvement in the interpretation results.
\end{remark}

\section{Experiments}
\label{sec:experiment}


\begin{table*}[!t]
    \vspace{-30pt}
    \caption{Evaluation of faithfulness metrics (Deletion, Insertion AUC scores, and average highest score) and location metrics (Point Game and Energy Point Game) on the MS-COCO, RefCOCO, and LVIS V1 (rare) validation sets for correctly detected or grounded samples using Grounding DINO.}\vspace{-18pt}
    \label{faithfulness_on_ms_coco}
    \begin{center}
        \resizebox{\textwidth}{!}{
            \begin{tabular}{ll|ccccccc|cc}
                \toprule
                \multirow{2}{*}{Datasets} & \multirow{2}{*}{Methods}  & \multicolumn{7}{c}{Faithfulness Metrics} & \multicolumn{2}{c}{Location Metrics} \\ \cmidrule(lr){3-9} \cmidrule(l){10-11}
                & & Ins. ($\uparrow$) & Del. ($\downarrow$) & Ins. (class) ($\uparrow$) & Del. (class) ($\downarrow$) & Ins. (IoU) ($\uparrow$) & Del. (IoU) ($\downarrow$) & Ave. high. score ($\uparrow$) & Point Game ($\uparrow$) & Energy PG ($\uparrow$) \\ \midrule
                \multirow{6}{*}{\begin{tabular}[c]{@{}l@{}}MS COCO~\cite{lin2014microsoft}\\ (Detection task)\end{tabular}} & Grad-CAM~\cite{selvaraju2020grad} & 0.2436 & 0.1526 & 0.3064 & 0.2006 & 0.6229 & 0.5324 & 0.5904 & 0.1746 & 0.1463 \\
                & SSGrad-CAM++~\cite{yamauchi2024spatial} & 0.2107 & 0.1778 & 0.2639 & 0.2314 & 0.5981 & 0.5511 & 0.5886 & 0.1905 & 0.1293 \\
                & D-RISE~\cite{petsiuk2021black} & 0.4412 & 0.0402 & 0.5081 & 0.0886 & 0.8396 & 0.3642 & 0.6215 & 0.9497 & 0.1850 \\
                & D-HSIC~\cite{novello2022making} & 0.3776 & 0.0439 & 0.4382 & 0.0903 & 0.8301 & 0.3301 & 0.5862 & 0.7328 & 0.1861 \\
                & ODAM~\cite{zhao2024gradient_detector} & 0.3103 & 0.0519 & 0.3655 & 0.0894 & 0.7869 & 0.3984 & 0.5865 & 0.5431 & 0.2034\\
                & \cellcolor[HTML]{EFEFEF}Ours & \cellcolor[HTML]{EFEFEF}\textbf{0.5459} & \cellcolor[HTML]{EFEFEF}\textbf{0.0375} & \cellcolor[HTML]{EFEFEF}\textbf{0.6204} & \cellcolor[HTML]{EFEFEF}\textbf{0.0882} & \cellcolor[HTML]{EFEFEF}\textbf{0.8581} & \cellcolor[HTML]{EFEFEF}\textbf{0.3300} & \cellcolor[HTML]{EFEFEF}\textbf{0.6873} & \cellcolor[HTML]{EFEFEF}\textbf{0.9894} & \cellcolor[HTML]{EFEFEF}\textbf{0.2046} \\ 
                \midrule
                \multirow{6}{*}{\begin{tabular}[c]{@{}l@{}}RefCOCO~\cite{kazemzadeh2014referitgame}\\ (REC task)\end{tabular}} & Grad-CAM~\cite{selvaraju2020grad} & 0.3749 & 0.4237 & 0.4658 & 0.5194 & 0.7516 & 0.7685 & 0.7481 & 0.2380 & 0.2171 \\
                & SSGrad-CAM++~\cite{yamauchi2024spatial} & 0.4113 & 0.3925 & 0.5008 & 0.4851 & 0.7700 & 0.7588 & 0.7561 & 0.2820 & 0.2262 \\
                & D-RISE~\cite{petsiuk2021black} & 0.6178 & 0.1605 & 0.7033 & 0.3396 & 0.8606 & 0.5164 & 0.8471 & 0.9400 & 0.2870 \\
                & D-HSIC~\cite{novello2022making} & 0.5491 & 0.1846 & 0.6295 & 0.3509 & 0.8504 & 0.5120 & 0.7739 & 0.7900 & 0.3190 \\
                & ODAM~\cite{zhao2024gradient_detector} & 0.4778 & 0.2718 & 0.5620 & 0.3757 & 0.8217 & 0.6641 & 0.7425 & 0.6320 & 0.3529 \\
                & \cellcolor[HTML]{EFEFEF}Ours & \cellcolor[HTML]{EFEFEF}\textbf{0.7419} & \cellcolor[HTML]{EFEFEF}\textbf{0.1250} & \cellcolor[HTML]{EFEFEF}\textbf{0.8080} & \cellcolor[HTML]{EFEFEF}\textbf{0.2457} & \cellcolor[HTML]{EFEFEF}\textbf{0.9050} & \cellcolor[HTML]{EFEFEF}\textbf{0.5103} & \cellcolor[HTML]{EFEFEF}\textbf{0.8842} & \cellcolor[HTML]{EFEFEF}\textbf{0.9460} & \cellcolor[HTML]{EFEFEF}\textbf{0.3566} \\
                \midrule
                \multirow{6}{*}{\begin{tabular}[c]{@{}l@{}}LVIS V1 (rare)~\cite{gupta2019lvis}\\ (Zero-shot det. task)\end{tabular}} & Grad-CAM~\cite{selvaraju2020grad} & 0.1253 & 0.1294 & 0.1801  & 0.1814 & 0.5657 & 0.5910 & 0.3549 & 0.1151 & 0.0941 \\
                & SSGrad-CAM++~\cite{yamauchi2024spatial} & 0.1253 & 0.1254 & 0.1765 & 0.1775 & 0.5800 & 0.5691 & 0.3504 & 0.1091 & 0.0931 \\
                & D-RISE~\cite{petsiuk2021black} & 0.2808 & 0.0289 & 0.3348 & 0.0835 & 0.8303 & 0.3174 & 0.4289 & 0.9697 & 0.1462 \\
                & D-HSIC~\cite{novello2022making} & 0.2417 & 0.0353 & 0.2912 & 0.0928 & 0.8187 & 0.3550 & 0.4044 & 0.8303 & 0.1730 \\
                & ODAM~\cite{zhao2024gradient_detector} & 0.2009 & 0.0410 & 0.2478 & 0.0844 & 0.7760 & 0.4082 & 0.3694 & 0.6061 & \textbf{0.2050} \\
                & \cellcolor[HTML]{EFEFEF}Ours & \cellcolor[HTML]{EFEFEF}\textbf{0.3695} & \cellcolor[HTML]{EFEFEF}\textbf{0.0277} & \cellcolor[HTML]{EFEFEF}\textbf{0.4275} & \cellcolor[HTML]{EFEFEF}\textbf{0.0799} & \cellcolor[HTML]{EFEFEF}\textbf{0.8479} & \cellcolor[HTML]{EFEFEF}\textbf{0.3242} & \cellcolor[HTML]{EFEFEF}\textbf{0.4969} & \cellcolor[HTML]{EFEFEF}\textbf{0.9758} & \cellcolor[HTML]{EFEFEF}0.1785 \\
                \bottomrule
                \end{tabular}
        }
    \end{center}
    \vspace{-16pt}
\end{table*}

\begin{figure*}[h]
    \centering
    \includegraphics[width=\textwidth]{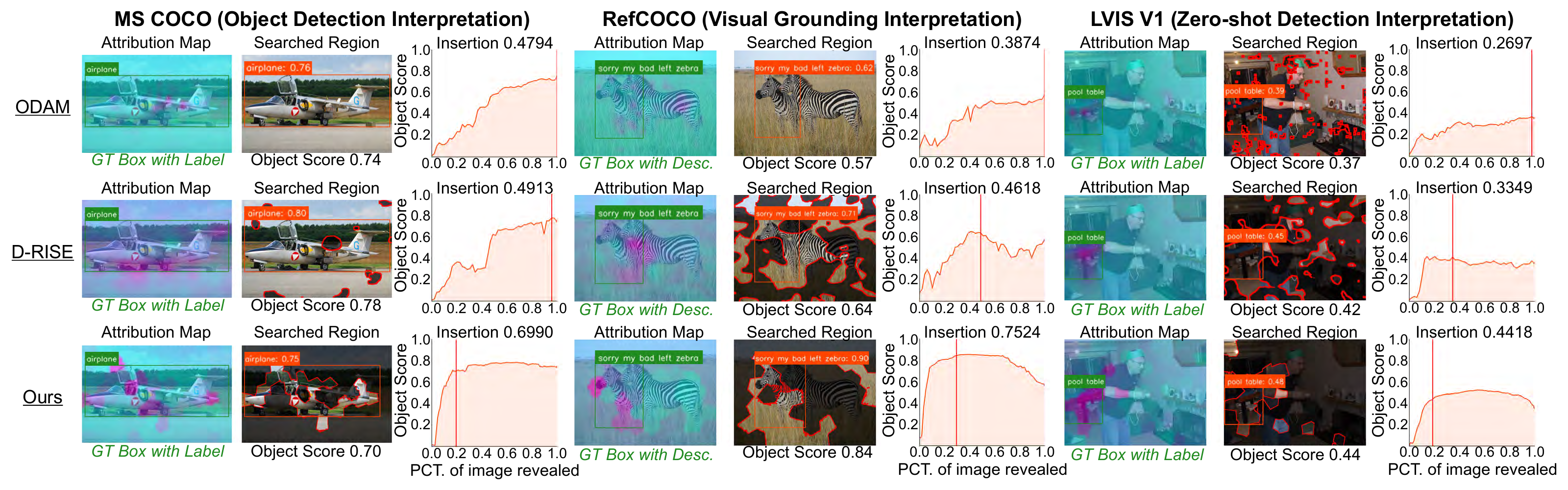}\vspace{-10pt} 
    \caption{Visualization results of Grounding DINO for interpreting object localization tasks on the MS COCO, RefCOCO, and LVIS datasets. 
    The first column shows the saliency map with the ground truth box and label. The second column presents detection within the limited search region, and the third columns display the insertion curves.
    } 
    \label{vis:groundingdino}
    \vspace{-10pt}
\end{figure*}

\subsection{Experimental Setup}\label{experimental_setup}

\textbf{Datasets.} We evaluate the proposed method on three object-level datasets: MS COCO 2017~\cite{lin2014microsoft}, LVIS V1~\cite{gupta2019lvis}, and RefCOCO~\cite{kazemzadeh2014referitgame}. The MS COCO dataset, used for interpreting object detection with 80 classes, involves random sampling of 5 correctly detected (IoU $> 0.5$, correct category), 3 misclassified (IoU $> 0.5$, wrong category), and 3 undetected (IoU $< 0.5$ or low confidence) instances per class. The LVIS V1 dataset, with 1,203 categories including 337 rare ones (1–10 images), uses Grounding DINO for interpreting zero-shot detection on rare classes, analyzing 165 correctly detected, 152 misclassified, and 534 undetected samples. The RefCOCO dataset, for interpreting referring expression comprehension (REC), evaluates 500 correctly grounded (IoU $> 0.5$) and 200 incorrectly grounded (IoU $< 0.5$) samples per model.

\textbf{Baselines.} The interpretation methods we compare include gradient-based approaches: Grad-CAM~\cite{selvaraju2020grad}, SSGrad-CAM++\cite{yamauchi2024spatial}, and ODAM\cite{zhao2024gradient_detector}, as well as perturbation-based methods: D-RISE~\cite{petsiuk2021black} and D-HSIC~\cite{novello2022making}. D-RISE and D-HSIC are allowed to use all model outputs, including low-confidence detection boxes, to enhance the sufficiency of saliency estimation.

\textbf{Implementation Details.} We validate our method on two object-level foundation models, Grounding DINO~\cite{liu2023grounding} and Florence-2~\cite{xiao2024florence}, across object detection and referring expression comprehension (REC) tasks. To interpret model decisions, we set the ground truth bounding box and category as target box $\boldsymbol{b}_{\text{target}}$ and category $c$. By default, the SLICO superpixel segmentation algorithm divides the image into 100 sparse sub-regions.

\subsection{Evaluation Metrics}\label{evaluation}


\textbf{Faithfulness Metrics.} We evaluate the faithfulness of interpretation results across the entire image using Insertion and Deletion scores from the detection task, following the setup by Petsiuk \textit{et al.}~\cite{petsiuk2021black}. These metrics reflect location and recognition accuracy. We also calculate classification and IoU scores for the box closest to the target explanation denoted as Insertion (class), Deletion (class), Insertion (IoU), and Deletion (IoU). Additionally, we compute the average highest confidence, defined as the highest class response for the box closest to the target with IoU $> 0.5$. For failure samples, we introduce the Explaining Successful Rate (ESR) to evaluate whether the saliency map can effectively guide the search within a limited region, enabling the model to correctly detect misclassified or low-confidence objects (with IoU $> 0.5$ and classification confidence above the detector’s default threshold) previously.

\textbf{Location Metrics.} We report the Point Game~\cite{zhang2018top} and Energy Point Game~\cite{wang2020score} metrics. Since these metrics typically assume that the model’s decision is relevant to the target object~\cite{petsiuk2021black}, we calculate them only for interpretation results on samples the model has correctly detected.

\subsection{Faithfulness Analysis}\label{faithfulness_analysis}


\subsubsection{Faithfulness on Grounding DINO}

We begin by validating our approach in explaining decisions across various object-level tasks using Grounding DINO~\cite{liu2023grounding}. Table~\ref{faithfulness_on_ms_coco} presents the faithfulness results across different datasets and tasks, demonstrating the strong interpretative faithfulness of our method.

In the object detection interpretation task on the MS COCO dataset, our method achieved state-of-the-art results, outperforming the D-RISE method by 23.7\% on the Insertion metric, 6.7\% on the Deletion metric, and 10.6\% on the average highest score. The Insertion and Deletion metrics also achieved top performance under both category and IoU conditions. Additionally, our method performed best in location metrics, indicating state-of-the-art performance in explaining general object detection tasks. We also found that gradient-based methods, such as ODAM, exhibit low faithfulness in explaining Grounding DINO decisions. This is primarily due to the fusion of text and visual modalities, which causes the gradient to be influenced by both, reducing its effectiveness in attributing to the visual modality alone.

In the referring expression comprehension (REC) interpretation task on the RefCOCO dataset, our method achieved state-of-the-art results across all evaluation metrics, surpassing the D-RISE method by 20.1\% on the Insertion metric, 22.1\% on the Deletion metric, and 4.4\% on the average highest score. These results confirm that our method maintains high faithfulness in explaining object description understanding.


We also interpret zero-shot object detection for rare classes on the LVIS V1 dataset. Our method achieved SOTA results in faithfulness metrics, surpassing D-RISE by 31.6\% on the Insertion metric, 4.2\% on the Deletion metric, and 15.9\% on the average highest score. Although our method ranks second to ODAM on Energy Point Game metrics, this is mainly due to Grounding DINO not being fully trained on rare LVIS classes; as a result, model responses may be influenced by other regions rather than the object.

Experimental results show that our method significantly outperforms other interpretability approaches, primarily due to the submodularity of our objective function, enabling precise evaluation of critical sub-regions. As shown in Figure~\ref{vis:groundingdino}, ODAM’s saliency maps are diffuse, and D-RISE’s maps are noisy, whereas our method sharply highlights essential sub-regions, capturing edges and class-specific features that enhance interpretability.


\subsubsection{Faithfulness on Florence-2}



We further validated our approach on Florence-2~\cite{xiao2024florence}, a foundation model lacking object confidence scores. This limitation requires our search to rely solely on IoU as a guiding metric, excluding comparison with gradient-based methods. As shown in Table~\ref{faithfulness_on_florence-2}, our method achieves state-of-the-art performance.

In the object detection interpretability task on the MS COCO dataset, our method outperforms D-RISE by 3.8\% on the Insertion metric and achieves a substantial improvement of 50.7\% on the Deletion metric. Additionally, our method enhances the Point Game and Energy Point Game metrics by 8.3\% and 60.7\%, respectively. Due to the absence of object confidence scores in Florence-2, it becomes challenging to identify areas that quickly invalidate the model’s detection results. Compared to perturbation-based methods, our approach offers significant advantages; it relies more directly on the object itself, resulting in stronger faithfulness to the model’s detection behavior.


We also interpret referring expression comprehension (REC) interpretation task on the RefCOCO dataset, outperforming D-RISE by 6.1\% and 66.9\% in Insertion and Deletion metrics, respectively. Our approach also achieved state-of-the-art localization results, with a 14.6\% improvement in the Energy Point Game, demonstrating strong interpretability in accurately pinpointing object information.

We visualized some results of our method in Figure~\ref{vis:florence-2}, illustrating that it effectively identifies key information about the object’s location and categorical description. While D-RISE performed well on the Insertion metric, its attribution results were diffuse and lacked reliability, and it did not perform as well on the Deletion metric.

\begin{table}[!t]
    \caption{Evaluation of faithfulness metrics (Deletion and Insertion AUC scores) and location metrics (Point Game and Energy Point Game) on the MS COCO and RefCOCO validation sets for correctly detected and grounded samples using Florence-2.}\vspace{-18pt}
    \label{faithfulness_on_florence-2}
    \begin{center}
        \resizebox{0.48 \textwidth}{!}{
            \begin{tabular}{ll|cc|cc}
                \toprule
                \multirow{2}{*}{Datasets} & \multirow{2}{*}{Methods} & \multicolumn{2}{c}{Faithfulness Metrics} & \multicolumn{2}{c}{Location Metrics} \\ \cmidrule(lr){3-4} \cmidrule(l){5-6}
                & & Insertion ($\uparrow$) & Deletion ($\downarrow$) & Point Game ($\uparrow$) & Energy PG ($\uparrow$) \\ \midrule
                \multirow{3}{*}{\begin{tabular}[c]{@{}l@{}}MS COCO~\cite{lin2014microsoft}\\ (Detection task)\end{tabular}} 
                & D-RISE~\cite{petsiuk2021black} & 0.7477 & 0.0972 & 0.8850 & 0.1568 \\
                & D-HSIC~\cite{novello2022making} & 0.5345 & 0.2730 & 0.2925 & 0.0862 \\
                & \cellcolor[HTML]{EFEFEF}Ours & \cellcolor[HTML]{EFEFEF}\textbf{0.7759} & \cellcolor[HTML]{EFEFEF}\textbf{0.0479} & \cellcolor[HTML]{EFEFEF}\textbf{0.9583} & \cellcolor[HTML]{EFEFEF}\textbf{0.2519} \\ \midrule
                \multirow{3}{*}{\begin{tabular}[c]{@{}l@{}}RefCOCO~\cite{kazemzadeh2014referitgame}\\ (REC task)\end{tabular}} & D-RISE~\cite{petsiuk2021black} & 0.7922 & 0.3505 & 0.8480 & 0.2464 \\
                & D-HSIC~\cite{novello2022making} & 0.7639 & 0.3560 & 0.6980 & 0.2754 \\
                & \cellcolor[HTML]{EFEFEF}Ours & \cellcolor[HTML]{EFEFEF}\textbf{0.8409} & \cellcolor[HTML]{EFEFEF}\textbf{0.1159} & \cellcolor[HTML]{EFEFEF}\textbf{0.8660} & \cellcolor[HTML]{EFEFEF}\textbf{0.3927} \\
                \bottomrule
                \end{tabular}
        }
    \end{center}
    \vspace{-16pt}
\end{table}

\begin{figure}[!t]
    \centering
    \includegraphics[width=0.48\textwidth]{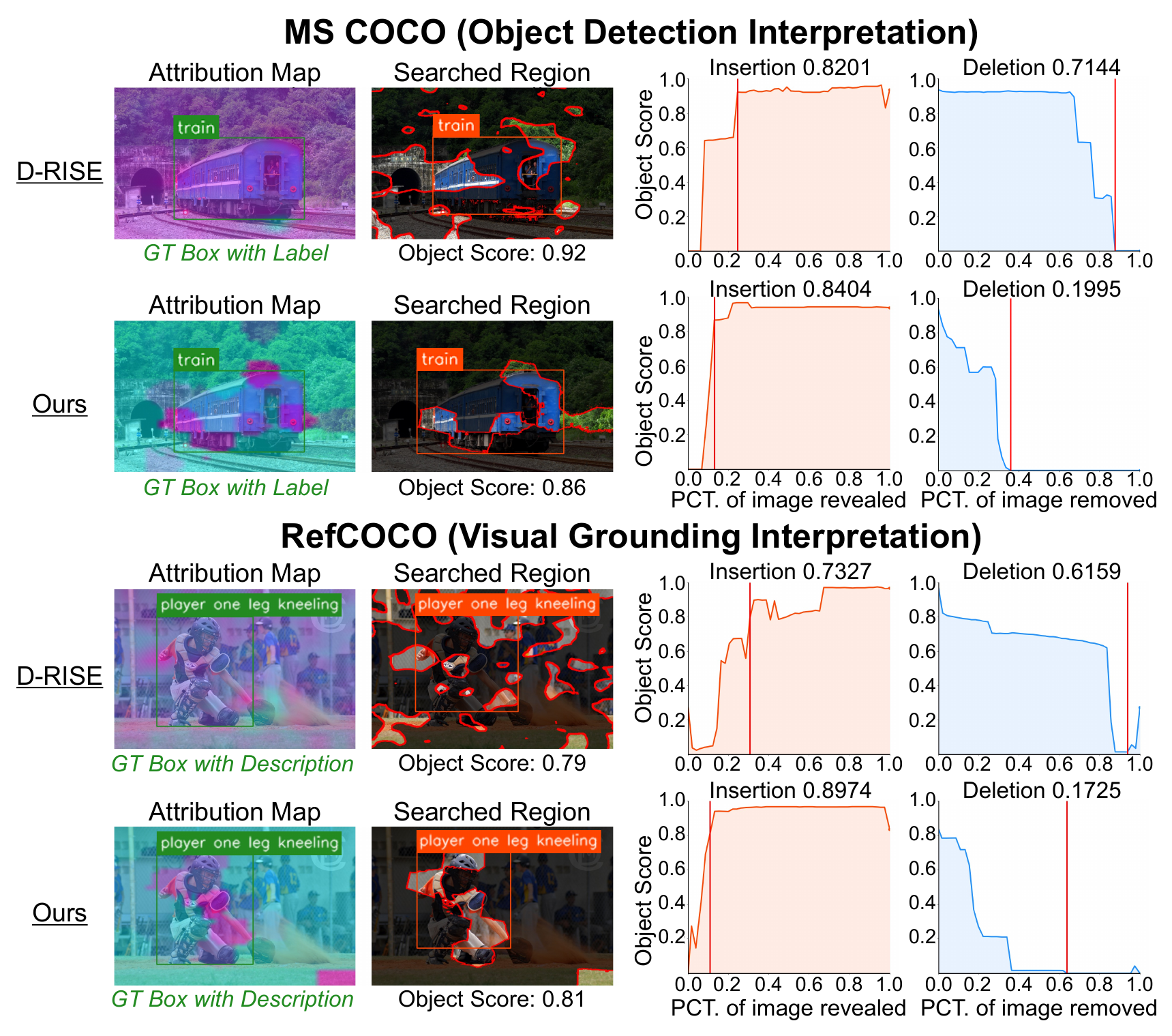}\vspace{-10pt} 
    \caption{Visualization results of Florence-2 for interpreting object localization tasks on the MS COCO and RefCOCO datasets. 
    }  
    \label{vis:florence-2}
    \vspace{-12pt}
\end{figure}

\subsection{Interpreting Failures in Object-Level Tasks}\label{explaining_error}


\subsubsection{Interpreting REC Failures}


Failures in the REC task primarily arise from visual grounding errors, often due to interference from other objects in the scene that become entangled with the text. Removing this distracting visual information can clarify grounding results and improve error understanding at the visual input level. To achieve this, we specify the correct location as the target box, denoted as $\boldsymbol{b}_{\text{target}}$, for input attribution. Table~\ref{failures_on_refcoco} shows the attribution results, where gradient-based methods exhibit limited performance in faithfulness metrics. Compared to D-RISE, our approach improves Insertion by 42.9\%, class score-based Insertion by 25.1\%, and average highest score by 14.9\%. Figure~\ref{vis:refcoco-mistake} illustrates model error interpretations, with the cyan-highlighted region indicating decision errors due to visual input interference that leads the foundation model astray.

\subsubsection{Interpreting Detection Failures}

Detection failures include misclassification and undetected objects. While previous works~\cite{petsiuk2021black,zhao2024gradient_detector} used visualization to explain errors, it lacked quantitative analysis, resulting in lower trustworthiness. We used saliency maps to identify input-level interference causing detection failures.

\textbf{Misclassified Interpretation.} As shown in Table~\ref{misclassification_on_ms_coco}, On the MS COCO dataset, our method outperforms D-RISE, improving Insertion by 54.7\%, Insertion (class) by 49.1\%, average highest score by 27.4\%, and explaining success rate (ESR) by 19.47\%. Similarly, on the LVIS dataset, our method shows advantages with improvements of 42.7\% in Insertion, 33.0\% in Insertion (class), 24.8\% in average highest score, and 24.34\% in ESR. This improvement is primarily due to our proposed Visual Precision Search, which confines attribution to the target negative response region, effectively refining the model’s detection results. Figure~\ref{vis:detection-misclassified} shows that the background surrounding the object interferes with the model’s decision-making. Improving the model by refining the contextual relationship between foreground and background may be a promising direction.

\begin{table}[!t]
    \caption{Insertion AUC scores and the average highest score on the RefCOCO validation sets for or the samples with incorrect localization in visual grounding using Grounding DINO.}\vspace{-18pt}
    \label{failures_on_refcoco}
    \begin{center}
        \resizebox{0.45\textwidth}{!}{
            \begin{tabular}{ll|ccc}
            \toprule
            \multirow{2}{*}{Datasets} & \multirow{2}{*}{Methods} & \multicolumn{3}{c}{Faithfulness Metrics} \\ \cmidrule(l){3-5} 
            & & Ins. ($\uparrow$) & Ins. (class) ($\uparrow$) & Ave. high. score ($\uparrow$) \\ \midrule
            \multirow{6}{*}{\begin{tabular}[c]{@{}l@{}}RefCOCO~\cite{kazemzadeh2014referitgame}\\ (REC task)\end{tabular}} & Grad-CAM~\cite{selvaraju2020grad} & 0.1536 & 0.2794 & 0.3295 \\
            & SSGrad-CAM++~\cite{yamauchi2024spatial} & 0.1590 & 0.2837 & 0.3266 \\
            & D-RISE~\cite{petsiuk2021black} & 0.3486 & 0.4787 & 0.6096 \\
            & D-HSIC~\cite{novello2022making} & 0.2274 & 0.3488 & 0.4495 \\
            & ODAM~\cite{zhao2024gradient_detector} & 0.1793 & 0.3001 & 0.3453 \\
            & \cellcolor[HTML]{EFEFEF}Ours & \cellcolor[HTML]{EFEFEF}\textbf{0.4981} & \cellcolor[HTML]{EFEFEF}\textbf{0.5990} & \cellcolor[HTML]{EFEFEF}\textbf{0.7007} \\ 
            \bottomrule
            \end{tabular}
        }
    \end{center}
    \vspace{-20pt}
\end{table}

\begin{figure}[!t]
    \centering
    \includegraphics[width=0.48\textwidth]{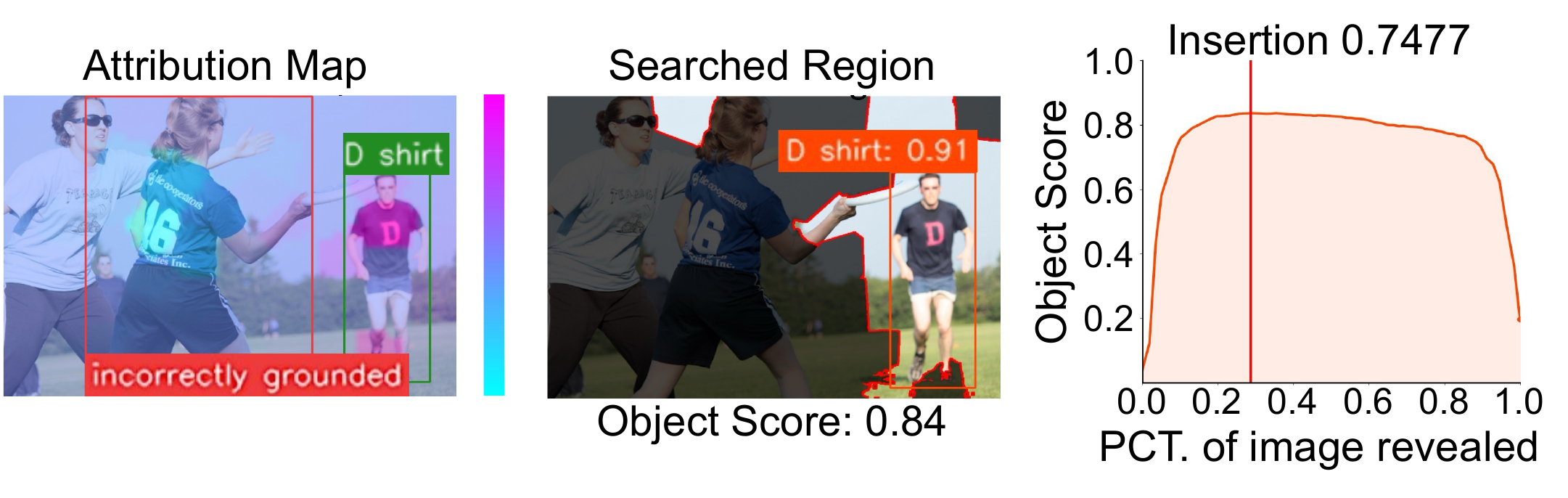}\vspace{-14pt} 
    \caption{Visualization of the method for discovering what causes the Grounding DINO incorrectly grounded on RefCOCO. 
    }  
    \label{vis:refcoco-mistake} 
    \vspace{-12pt}
\end{figure}

\begin{table}[!t]
    \caption{Insertion AUC scores, average highest score, and explaining successful rate (ESR) on the MS-COCO and the LVIS validation sets for misclassified samples using Grounding DINO.}\vspace{-18pt}
    \label{misclassification_on_ms_coco}
    \begin{center}
        \resizebox{0.48\textwidth}{!}{
            \begin{tabular}{ll|cccc}
            \toprule
            \multirow{2}{*}{Datasets} & \multirow{2}{*}{Methods} & \multicolumn{4}{c}{Faithfulness Metrics} \\ \cmidrule(l){3-6} 
            & & Ins. ($\uparrow$) & Ins. (class) ($\uparrow$) & Ave. high. score ($\uparrow$) & ESR ($\uparrow$) \\ \midrule
            \multirow{6}{*}{\begin{tabular}[c]{@{}l@{}}MS COCO~\cite{lin2014microsoft}\\ (Detection task)\end{tabular}} & Grad-CAM~\cite{selvaraju2020grad} & 0.1091 & 0.1478 & 0.3102 & 38.38\% \\
            & SSGrad-CAM++~\cite{yamauchi2024spatial} & 0.0960 & 0.1336 & 0.2952 & 33.51\% \\
            & D-RISE~\cite{petsiuk2021black} & 0.2170 & 0.2661 & 0.3603 & 50.26\% \\
            & D-HSIC~\cite{novello2022making} & 0.1771 & 0.2161 & 0.3143 & 34.59\% \\
            & ODAM~\cite{zhao2024gradient_detector} & 0.1129 & 0.1486 & 0.2869 & 32.97\% \\
            & \cellcolor[HTML]{EFEFEF}Ours & \cellcolor[HTML]{EFEFEF}\textbf{0.3357} & \cellcolor[HTML]{EFEFEF}\textbf{0.3967} & \cellcolor[HTML]{EFEFEF}\textbf{0.4591} & \cellcolor[HTML]{EFEFEF}\textbf{69.73\%} \\ \midrule
            \multirow{6}{*}{\begin{tabular}[c]{@{}l@{}}LVIS V1 (rare)~\cite{gupta2019lvis}\\ (Zero-shot det. task)\end{tabular}} & Grad-CAM~\cite{selvaraju2020grad} & 0.0503 & 0.0891 & 0.1564 & 12.50\% \\
            & SSGrad-CAM++~\cite{yamauchi2024spatial} & 0.0574 & 0.0946 & 0.1580 & 11.84\% \\
            & D-RISE~\cite{petsiuk2021black} & 0.1245 & 0.1647 & 0.2088 & 28.95\% \\
            & D-HSIC~\cite{novello2022making} & 0.0963 & 0.1247 & 0.1748 & 16.45\% \\
            & ODAM~\cite{zhao2024gradient_detector} & 0.0575 & 0.0954 & 0.1520 & 9.21\% \\
            & \cellcolor[HTML]{EFEFEF}Ours & \cellcolor[HTML]{EFEFEF}\textbf{0.1776} & \cellcolor[HTML]{EFEFEF}\textbf{0.2190} & \cellcolor[HTML]{EFEFEF}\textbf{0.2606} & \cellcolor[HTML]{EFEFEF}\textbf{53.29\%} \\
            \bottomrule
            \end{tabular}
        }
    \end{center}
    \vspace{-20pt}
\end{table}

\begin{figure}[!t]
    \centering
    \includegraphics[width=0.48\textwidth]{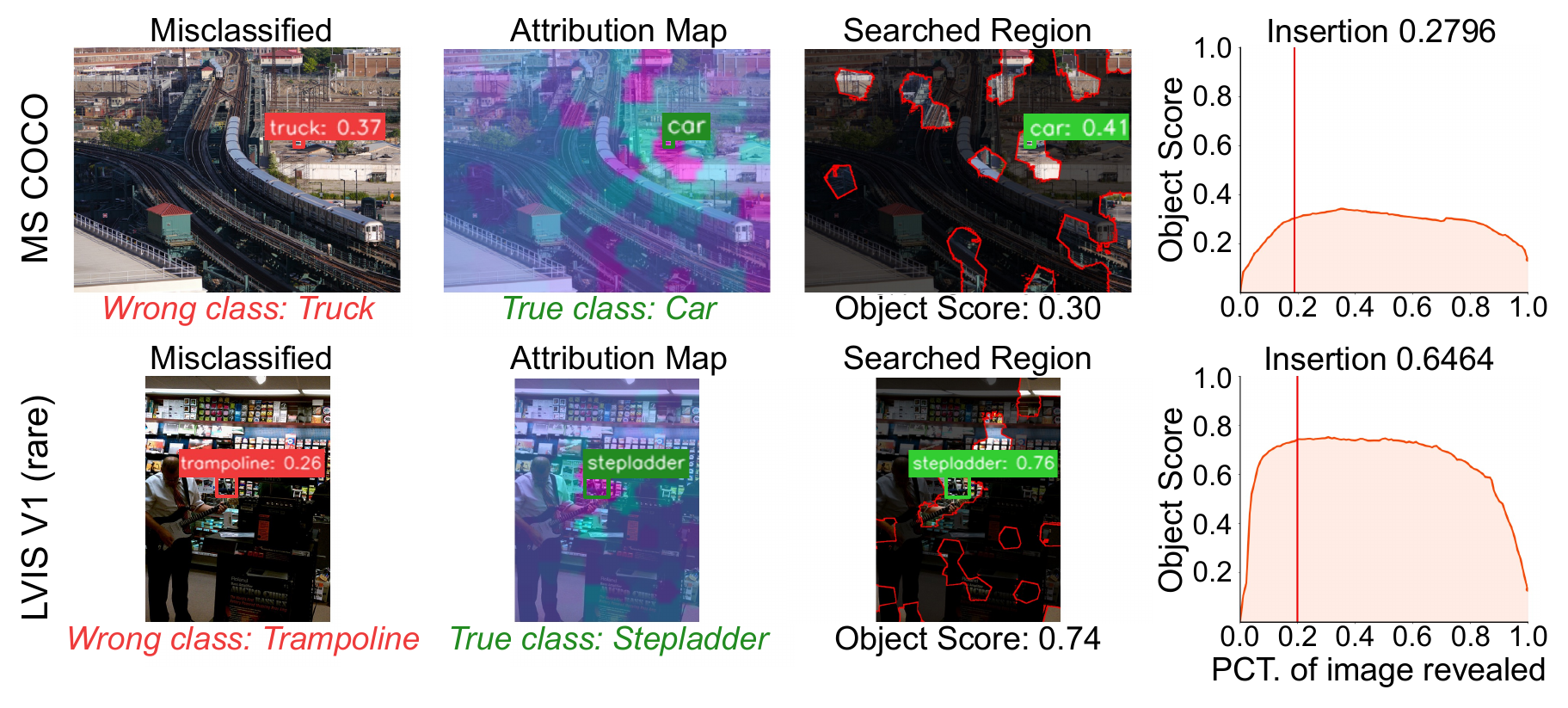}\vspace{-12pt} 
    \caption{Visualization of our method reveals the causes of Grounding DINO misclassifications on MS COCO and LVIS. The cyan region in the saliency map highlights the regions responsible for the model’s misclassification.
    }  
    \label{vis:detection-misclassified} 
    \vspace{-16pt}
\end{figure}

\begin{table}[!t]
    \vspace{-16pt}
    \caption{Insertion, average highest score, and explaining successful rate (ESR) on the MS-COCO and the LVIS V1 (rare) validation sets for missed detection samples using Grounding DINO.}\vspace{-22pt}
    \label{misdetection_on_ms_coco}
    \begin{center}
        \resizebox{0.48\textwidth}{!}{
            \begin{tabular}{ll|cccc}
            \toprule
            \multirow{2}{*}{Datasets} & \multirow{2}{*}{Methods} & \multicolumn{4}{c}{Faithfulness Metrics}          \\ \cmidrule(l){3-6} 
            & & Ins. ($\uparrow$) & Ins. (class) ($\uparrow$) & Ave. high. score ($\uparrow$) & ESR ($\uparrow$) \\ \midrule
            \multirow{6}{*}{\begin{tabular}[c]{@{}l@{}}MS COCO~\cite{lin2014microsoft}\\ (Detection task)\end{tabular}} & Grad-CAM~\cite{selvaraju2020grad} & 0.0760 & 0.1321 & 0.2153 & 16.44\% \\
            & SSGrad-CAM++~\cite{yamauchi2024spatial} & 0.0671 & 0.1151 & 0.2124 & 16.44\% \\
            & D-RISE~\cite{petsiuk2021black} & 0.1538 & 0.2260 & 0.2564 & 26.94\% \\
            & D-HSIC~\cite{novello2022making} & 0.1101 & 0.1716 & 0.1945 & 13.56\% \\
            & ODAM~\cite{zhao2024gradient_detector} & 0.0745 & 0.1350 & 0.2037 & 13.78\% \\
            & \cellcolor[HTML]{EFEFEF}Ours & \cellcolor[HTML]{EFEFEF}\textbf{0.2102} & \cellcolor[HTML]{EFEFEF}\textbf{0.3011} & \cellcolor[HTML]{EFEFEF}\textbf{0.3014} & \cellcolor[HTML]{EFEFEF}\textbf{41.33\%} \\ \midrule
            \multirow{6}{*}{\begin{tabular}[c]{@{}l@{}}LVIS V1 (rare)~\cite{gupta2019lvis}\\ (Zero-shot det. task)\end{tabular}} & Grad-CAM~\cite{selvaraju2020grad} & 0.0291 & 0.0689 & 0.0901 & 5.43\% \\
            & SSGrad-CAM++~\cite{yamauchi2024spatial} & 0.0292 & 0.0680 & 0.0897 & 5.24\% \\
            & D-RISE~\cite{petsiuk2021black} & 0.0703 & 0.1184 & 0.1312 & 18.73\% \\
            & D-HSIC~\cite{novello2022making} & 0.0516 & 0.0920 & 0.1168 & 13.48\% \\
            & ODAM~\cite{zhao2024gradient_detector} & 0.0283 & 0.0716 & 0.0851 & 4.68\% \\
            & \cellcolor[HTML]{EFEFEF}Ours & \cellcolor[HTML]{EFEFEF}\textbf{0.1155} & \cellcolor[HTML]{EFEFEF}\textbf{0.1886} & \cellcolor[HTML]{EFEFEF}\textbf{0.1784} & \cellcolor[HTML]{EFEFEF}\textbf{30.15\%} \\
            \bottomrule
            \end{tabular}
        }
    \end{center}
\end{table}

\begin{figure}[!t]
    \centering
    \includegraphics[width=0.48\textwidth]{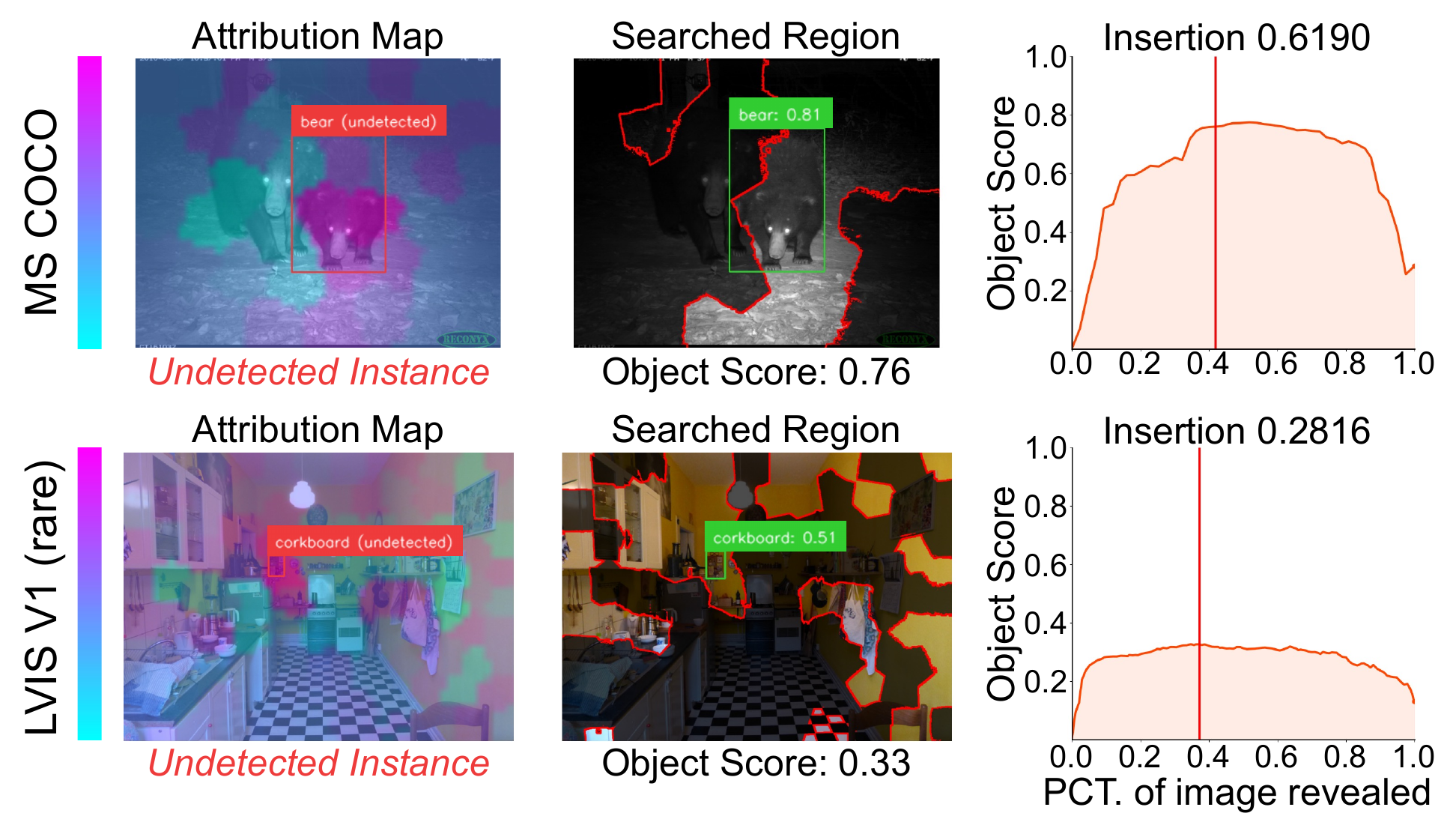}\vspace{-12pt} 
    \caption{Visualization of our method reveals the causes of Grounding DINO undetected on MS COCO and LVIS. The cyan region in the saliency map highlights the regions responsible for the model’s detection failure.
    }  
    \label{vis:detection-undetected} 
    \vspace{-16pt}
\end{figure}

\textbf{Undetected Interpretation.} Low object confidence may result from both the model’s feature representation and confounding input-level factors. We analyze causes of missed detections, with Table~\ref{misdetection_on_ms_coco} showing that our method achieves SOTA performance across metrics. On MS COCO, it outperforms D-RISE with a 36.7\% improvement in Insertion and a 14.39\% gain in ESR, while on LVIS, it improves Insertion by 64.3\% and ESR by 11.42\%. Figure~\ref{vis:detection-undetected} visualizes our method’s explanations for undetected instances, revealing that errors may arise from challenges in distinguishing similar objects (e.g., the bear in the first row) and environmental influence on detection (e.g., the corkboard in the second row). These insights highlight current model limitations, offering directions for improvement.

\subsection{Ablation Study}\label{ablation_study}


\textbf{Ablation of the Submodular Function.} We analyze the effectiveness of our submodular function design on Grounding DINO using correctly detected samples from MS COCO. Table~\ref{ablation_function} shows the results. The Clue Score identifies regions beneficial for detection, improving Insertion and average highest score, while the Collaboration Score pinpoints sensitive regions, enhancing the Deletion metric by causing the detector to fail more quickly. Combining these scores enables our method to achieve optimal results across indicators, demonstrating the effectiveness of each score function within the submodular function.

\textbf{Ablation on Divided Sub-region Number.} The sub-region number affects search space quality. Using Grounding DINO, we assess its impact on correctly predicted and misclassified samples in MS COCO. Figures~\ref{ablation-subregion-number}A-B show that when the model predicts correctly, increasing the number of sub-regions leads to higher Insertion and average highest scores, indicating that finer divisions enhance the faithfulness of search results. Additionally, Figures~\ref{ablation-subregion-number}D-F show that, for misclassified samples, Insertion, average highest score, and ESR also improve as the number of sub-regions increases. Figure~\ref{ablation-subregion-number}C shows inference times for our search algorithm on an RTX 3090 compared to D-RISE across various subregion counts. Increasing sub-regions improves faithfulness but also rapidly increases inference time. 
Future work will aim to increase subregion count and inference speed without sacrificing attribution performance.

\begin{table}[!t]
    \vspace{-16pt}
    \caption{Ablation study on function score components for Grounding DINO on the MS COCO validation set.}\vspace{-16pt}
    \label{ablation}
    \begin{center}
    \resizebox{0.42\textwidth}{!}{
        \begin{tabular}{cc|ccc}
        \toprule
        Clue Score & Colla. Score & \multicolumn{3}{c}{Faithfulness Metrics}  \\ 
        (Eq.~\ref{clue_score}) & (Eq.~\ref{colla_score}) & Insertion ($\uparrow$)  & Deletion ($\downarrow$) & Ave. high. score ($\uparrow$) \\ \midrule
        \XSolidBrush & \Checkmark  & 0.3632 & \underline{0.0378} &  0.5967 \\
        \Checkmark & \XSolidBrush  & \underline{0.5370} & 0.0799 & \underline{0.6864} \\
        \Checkmark & \Checkmark    & \textbf{0.5459} & \textbf{0.0375} & \textbf{0.6873} \\  \bottomrule
        \end{tabular}
    }
    \end{center}
    \label{ablation_function}
    \vspace{-22pt}
\end{table}

\begin{figure}[!t]
    \centering
    \includegraphics[width=0.48\textwidth]{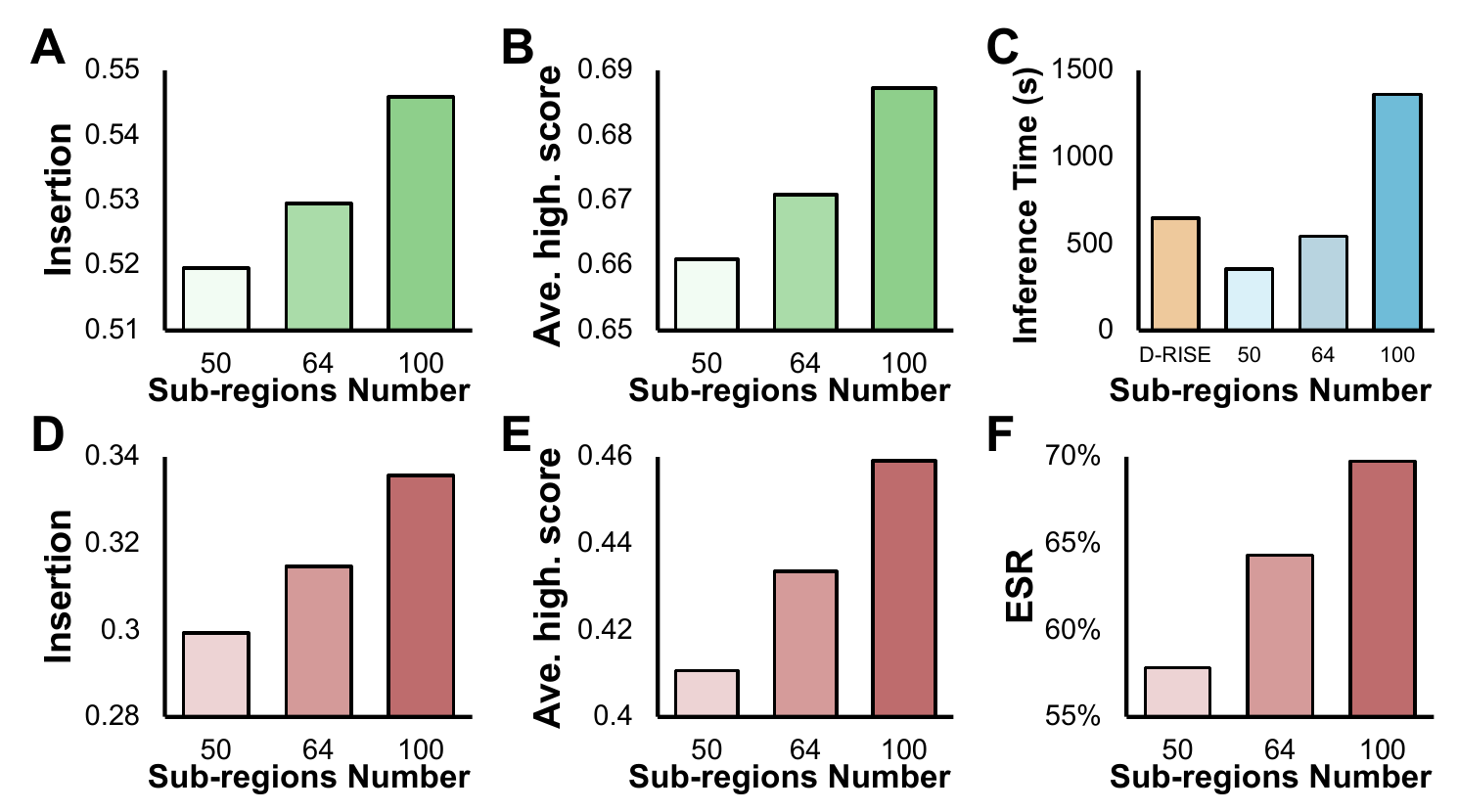}\vspace{-10pt} 
    \caption{Ablation on the number of sub-regions. For correct detections, panels \textbf{A} and \textbf{B} show the effect of sub-region count on Insertion and average highest score, respectively. Panel \textbf{C} presents the inference time for different sub-region counts. For misclassified samples, panels \textbf{D}, \textbf{E}, and \textbf{F} display the effects on Insertion, average highest score, and explanation success rate.
    } 
    \label{ablation-subregion-number} 
\end{figure}

\section{Conclusion}
In this paper, we propose an interpretable attribution method specifically tailored for object-level foundation models, called the Visual Precision Search method, which introduces a novel submodular mechanism that combines a clue score and a collaboration score. This method achieves enhanced interpretability with fewer search regions. Experiments on RefCOCO, MS COCO, and LVIS demonstrate that our approach improves object-level task interpretability over state-of-the-art methods for Grounding DINO and Florence-2 across various evaluation metrics. Furthermore, our method effectively interprets failures in visual grounding and object detection tasks.

\section*{Acknowledgements} This work was supported by the Shenzhen Science and Technology Program (No. KQTD20221101093559018), National Natural Science Foundation of China (No. 62372448, 62025604, 62306308, 62441619), Foundation of Key Laboratory of Education Informatization for Nationalities (YNNU), Ministry of Education (No. EIN2024B004).


{
    \small
    \bibliographystyle{ieeenat_fullname}
    \bibliography{main}
}

\appendix

\section*{\Huge Appendix}

\section{Proof of Theorem~\ref{submodular_properties} (Submodular Properties)}

\begin{proof}
Consider two sub-sets $S_{A}$ and $S_{B}$ in set $V$, where $S_{A} \subseteq S_{B} \subseteq V$. Given an element $\alpha$, where $\alpha = V \setminus S_{B}$. Let $\alpha = V \setminus S_{B}$ represent an element not in $S_{B}$. For the function $\mathcal{F}(\cdot)$  to satisfy the submodular property, the following necessary and sufficient conditions must hold diminishing returns,
\begin{equation}
    \mathcal{F}\left(S_{A} \cup \{ \alpha \}\right) - \mathcal{F}\left(S_{A} \right) \ge \mathcal{F}\left(S_{B} \cup \{ \alpha \}\right) - \mathcal{F}\left(S_{B} \right),
\end{equation}
and monotonic non-negative, $\mathcal{F}(S_A \cup \{\alpha\}) - \mathcal{F}(S) \ge 0$.

For Clue score (Eq.~\ref{clue_score}), let $f_{\text{cls}}(S) = \boldsymbol{s}_c$, $f_{\text{reg}}(S) = \boldsymbol{b}$, assuming that $f_{\text{cls}}$ and $f_{\text{reg}}(S)$ is differentiable in $S$, the individual element $\alpha$ of the collection division is relatively small, according to the Taylor decomposition~\cite{montavon2017explaining}, we can locally approximate $f_{\text{cls}} \left(S_{A} + \alpha\right) = f_{\text{cls}} \left( S_{A} \right) + \nabla{f_{\text{cls}}\left( S_{A} \right)} \cdot \alpha$, and $f_{\text{reg}} \left(S_{A} + \alpha\right) = f_{\text{reg}} \left( S_{A} \right) + \nabla{f_{\text{reg}}\left( S_{A} \right)} \cdot \alpha$. Since object detection can output many candidate boxes, we can regard the model output with added sub-elements as incremental output, that is, $f_{\text{cls}} \left(S_{A} + \alpha\right) = \boldsymbol{s}_c + \nabla{f_{\text{cls}}\left( S_{A} \right)} \cdot \alpha = \boldsymbol{s}_c + \boldsymbol{s}^{\ast}_c$, $f_{\text{reg}} \left(S_{A} + \alpha\right) = \boldsymbol{b} + \nabla{f_{\text{reg}}\left( S_{A} \right)} \cdot \alpha = \boldsymbol{b} + \boldsymbol{b}^{\ast}$. Assuming that the searched $\alpha$ is valid, i.e., $\nabla{f_{\text{reg}}}>0$ and $\nabla{f_{\text{cls}}}>0$. Thus:
\begin{equation}\label{clue_monotonic}
    \begin{aligned}
     & s_{\text{clue}}(S_{A}+\alpha, \boldsymbol{b}_{\text{target},c}) - s_{\text{clue}}(S_{A},\boldsymbol{b}_{\text{target}},c) \\
    =& \max_{\boldsymbol{b}_i \in f_{\text{reg}}(S_{A}+\alpha), s_{c,i}\in f_{\text{cls}}(S_{A}+\alpha)}\text{IoU}(\boldsymbol{b}_{\text{target}}, \boldsymbol{b}_i)\cdot s_{c,i} \\
    & - \max_{\boldsymbol{b}_i \in f_{\text{reg}}(S_{A}), s_{c,i}\in f_{\text{cls}}(S_{A})}\text{IoU}(\boldsymbol{b}_{\text{target}}, \boldsymbol{b}_i)\cdot s_{c,i} \\
    =&\max (s_{\text{clue}}(S_{A},\boldsymbol{b}_{\text{target}},c), \max_{\boldsymbol{b}_i \in \boldsymbol{b}^{\ast}, s_{c,i} \in \boldsymbol{s}^{\ast}_c}\text{IoU}(\boldsymbol{b}_{\text{target}}, \boldsymbol{b}_i) \cdot s_{c,i} )\\
    &-s_{\text{clue}}(S_{A},\boldsymbol{b}_{\text{target}},c)\\
    =&\max (0, \max_{\boldsymbol{b}_i \in \boldsymbol{b}^{\ast}, s_{c,i} \in \boldsymbol{s}^{\ast}_c} \nabla{f_{\text{reg}}\left( S_{A} \right)} \cdot \nabla{f_{\text{cls}}\left( S_{A} \right)} \cdot \alpha^{2})  \\
    \ge&0,
    \end{aligned}
\end{equation}
the clue score satisfies the monotonic non-negative property in the process of maximizing the marginal effect. Since $S_A \subseteq S_B \subseteq V$, for the model’s candidate boxes, $f_{\text{reg}}(S_B) > f_{\text{reg}}(S_A)$, then the range of the gain candidate box $\boldsymbol{b}_{A}^{\ast}$ that can be generated is $f_{\text{reg}}(V) - f_{\text{reg}}(S_A)$. After introducing the new element $\alpha$, a new candidate box with a gain $\boldsymbol{b}_{A}^{\ast} > \boldsymbol{b}_{B}^{\ast}$, closer to the target, can be generated. If both $S_A$ and $S_B$ contain positive subsets, then $\nabla{f_{\text{cls}}\left( S_{B} \right)}$ will become less severe or even disappear~\cite{sundararajan2017axiomatic}, thus, $\nabla{f_{\text{cls}}\left( S_{A} \right)} > \nabla{f_{\text{cls}}\left( S_{B} \right)}$. So we have:
\begin{equation}
    \begin{aligned}
    &\max_{\boldsymbol{b}_i \in \boldsymbol{b}_{A}^{\ast}, s_{c,i} \in \boldsymbol{s}^{\ast}_c} \nabla{f_{\text{reg}}\left( S_{A} \right)} \cdot \nabla{f_{\text{cls}}\left( S_{A} \right)} \cdot \alpha^{2} > \\
    &\max_{\boldsymbol{b}_i \in \boldsymbol{b}_{B}^{\ast}, s_{c,i} \in \boldsymbol{s}^{\ast}_c} \nabla{f_{\text{reg}}\left( S_{B} \right)} \cdot \nabla{f_{\text{cls}}\left( S_{B} \right)} \cdot \alpha^{2},
    \end{aligned}
\end{equation}
combining Eq.~\ref{clue_monotonic}, we have:
\begin{equation}\label{clue_submodular}
    \begin{aligned}
        &s_{\text{clue}}(S_{A}+\alpha, \boldsymbol{b}_{\text{target},c}) - s_{\text{clue}}(S_{A},\boldsymbol{b}_{\text{target}},c) > \\
        &s_{\text{clue}}(S_{B}+\alpha, \boldsymbol{b}_{\text{target},c}) - s_{\text{clue}}(S_{B},\boldsymbol{b}_{\text{target}},c).
    \end{aligned}
\end{equation}

\begin{table*}[]
    \caption{Evaluation of faithfulness metrics (Deletion, Insertion AUC scores, and average highest score) and location metrics (Point Game and Energy Point Game) on the MS-COCO validation set for correctly detected or grounded samples using traditional object detectors.}
    \label{faithfulness_on_traditional}
    \begin{center}
        \resizebox{\textwidth}{!}{
            \begin{tabular}{ll|cccccc|c}
                \toprule
                \multirow{2}{*}{Detectors} & \multirow{2}{*}{Methods}  & \multicolumn{6}{c}{Faithfulness Metrics} & \multicolumn{1}{c}{Location Metric} \\ \cmidrule(lr){3-8} \cmidrule(l){9-9}
                & & Ins. ($\uparrow$) & Del. ($\downarrow$) & Ins. (class) ($\uparrow$) & Del. (class) ($\downarrow$) & Ins. (IoU) ($\uparrow$) & Del. (IoU) ($\downarrow$) & Point Game ($\uparrow$) \\
                \midrule
                \multirow{4}{*}{\begin{tabular}[c]{@{}l@{}}Mask R-CNN~\cite{he2018mask}\\ (Two-stage)\end{tabular}} 
                & Grad-CAM~\cite{selvaraju2020grad} & 0.2657 & 0.2114 & 0.3746 & 0.3122 & 0.5348 & 0.4954 & 0.5554 \\
                & D-RISE~\cite{petsiuk2021black} & 0.6756 & 0.0814 & 0.7666 & 0.1570 & 0.8396 & 0.2987 & 0.8996 \\
                & ODAM~\cite{zhao2024gradient_detector} & 0.6067 & 0.0787 & 0.7218 & 0.1860 & 0.7890 & 0.3188 & 0.9934 \\
                & \cellcolor[HTML]{EFEFEF}Ours & \cellcolor[HTML]{EFEFEF}\textbf{0.7991} & \cellcolor[HTML]{EFEFEF}\textbf{0.0489} & \cellcolor[HTML]{EFEFEF}\textbf{0.8678} & \cellcolor[HTML]{EFEFEF}\textbf{0.1065} & \cellcolor[HTML]{EFEFEF}\textbf{0.8968} & \cellcolor[HTML]{EFEFEF}\textbf{0.2841} & \cellcolor[HTML]{EFEFEF}\textbf{0.9987} \\  
                \midrule
                \multirow{4}{*}{\begin{tabular}[c]{@{}l@{}}YOLO V3~\cite{redmon2018yolov3}\\ (One-stage)\end{tabular}} 
                & Grad-CAM~\cite{selvaraju2020grad} & 0.6283 & 0.2867 & 0.7961 & 0.4573 & 0.7271 & 0.5234 & 0.7268 \\ 
                & D-RISE~\cite{petsiuk2021black} & 0.7524 & 0.1889 & 0.8747 & 0.3629 & 0.8213 & 0.4587 & 0.8816 \\
                & ODAM~\cite{zhao2024gradient_detector} & 0.7329 & 0.2766 & 0.8943 & 0.4707 & 0.7936 & 0.5283 & 0.9838 \\ 
                & \cellcolor[HTML]{EFEFEF}Ours & \cellcolor[HTML]{EFEFEF}\textbf{0.8674} & \cellcolor[HTML]{EFEFEF}\textbf{0.1407} & \cellcolor[HTML]{EFEFEF}\textbf{0.9490} & \cellcolor[HTML]{EFEFEF}\textbf{0.3008} & \cellcolor[HTML]{EFEFEF}\textbf{0.8984} & \cellcolor[HTML]{EFEFEF}\textbf{0.3814} & \cellcolor[HTML]{EFEFEF}\textbf{0.9900} \\ 
                \midrule
                \multirow{4}{*}{\begin{tabular}[c]{@{}l@{}}FCOS~\cite{tian2020fcos}\\ (One-stage)\end{tabular}} 
                & Grad-CAM~\cite{selvaraju2020grad} & 0.2742 & 0.1417 & 0.3439 & 0.1845 & 0.6858 & 0.6176 & 0.5249 \\ 
                & D-RISE~\cite{petsiuk2021black} & 0.4421 & 0.0570 & 0.4968 & 0.1078 & 0.8578 & 0.3729 & 0.9193 \\
                & ODAM~\cite{zhao2024gradient_detector} & 0.4266 & 0.0497 & 0.4742 & 0.0853 & 0.8713 & 0.4212 & 0.9935 \\ 
                & \cellcolor[HTML]{EFEFEF}Ours & \cellcolor[HTML]{EFEFEF}\textbf{0.5746} & \cellcolor[HTML]{EFEFEF}\textbf{0.0414} & \cellcolor[HTML]{EFEFEF}\textbf{0.6301} & \cellcolor[HTML]{EFEFEF}\textbf{0.0815} & \cellcolor[HTML]{EFEFEF}\textbf{0.8900} & \cellcolor[HTML]{EFEFEF}\textbf{0.3698} & \cellcolor[HTML]{EFEFEF}\textbf{0.9980} \\ 
                \midrule
                \multirow{4}{*}{\begin{tabular}[c]{@{}l@{}}SSD~\cite{liu2016ssd}\\ (One-stage)\end{tabular}} 
                & Grad-CAM~\cite{selvaraju2020grad} & 0.3869 & 0.1466 & 0.4977 & 0.2022 & 0.6796 & 0.5366 & 0.7700 \\ 
                & D-RISE~\cite{petsiuk2021black} & 0.4882 & 0.0616 & 0.5722 & 0.0979 & 0.7852 & 0.4497 & 0.9243 \\
                & ODAM~\cite{zhao2024gradient_detector} & 0.5117 & 0.0913 & 0.6072 & 0.1416 & 0.7900 & 0.4801 & 0.9778 \\ 
                & \cellcolor[HTML]{EFEFEF}Ours & \cellcolor[HTML]{EFEFEF}\textbf{0.6891} & \cellcolor[HTML]{EFEFEF}\textbf{0.0483} & \cellcolor[HTML]{EFEFEF}\textbf{0.7594} & \cellcolor[HTML]{EFEFEF}\textbf{0.0835} & \cellcolor[HTML]{EFEFEF}\textbf{0.8700} & \cellcolor[HTML]{EFEFEF}\textbf{0.4366} & \cellcolor[HTML]{EFEFEF}\textbf{0.9941} \\
                \bottomrule
                \end{tabular}
        }
    \end{center}
\end{table*}

\begin{figure*}[!t]
    \centering
    \includegraphics[width=\textwidth]{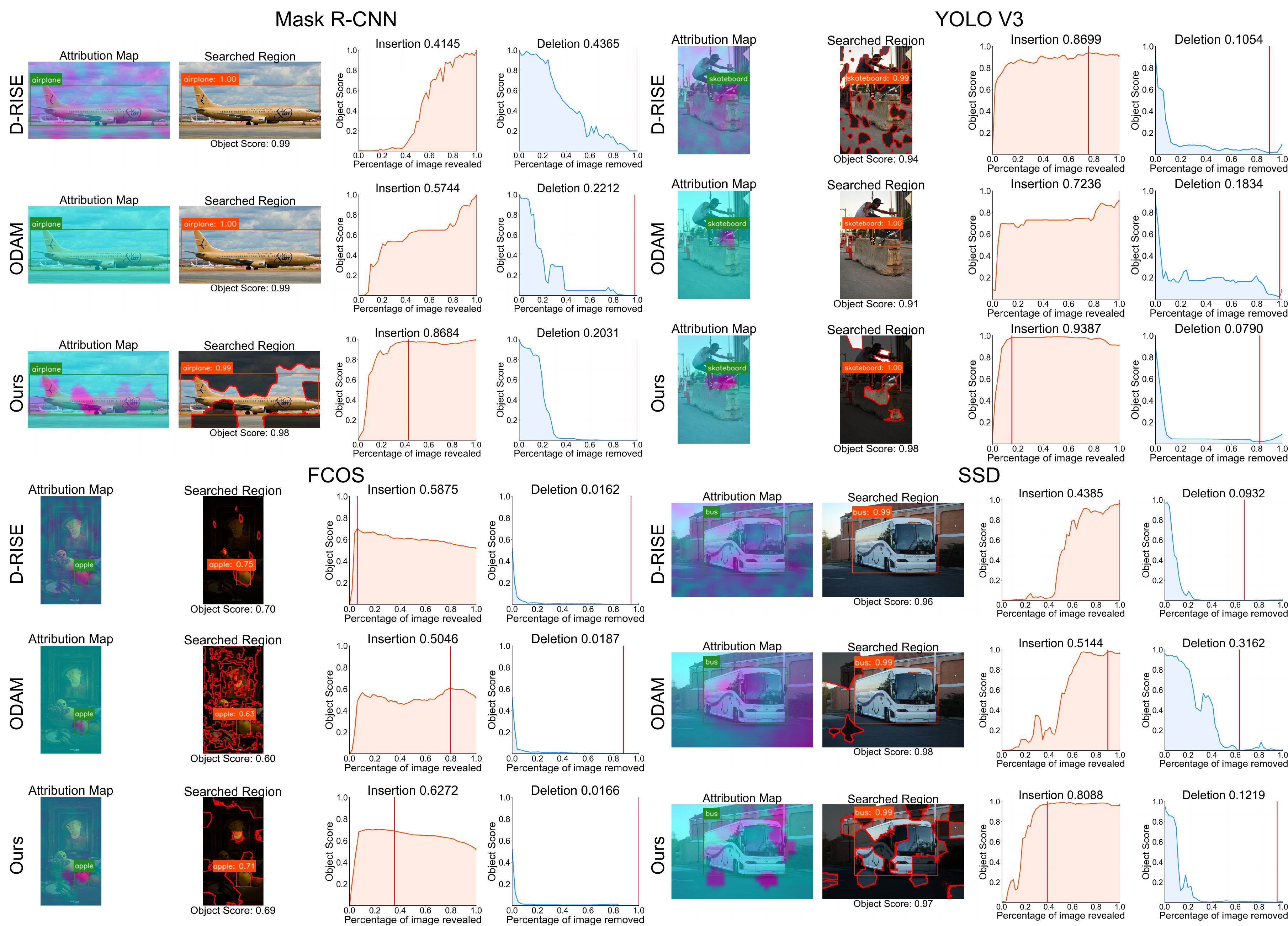}
    \caption{Visualization results of four object detectors for interpreting object detection task on the MS COCO dataset.} 
    \label{traditional_detector_visualization} 
\end{figure*}

Similar, for Collaboration score (Eq.~\ref{colla_score}), according to the Taylor decomposition~\cite{montavon2017explaining}, we can locally approximate $f_{\text{cls}} \left(V \setminus (S_{A} + \alpha) \right) = f_{\text{cls}} \left( V \setminus S_{A} \right) - \nabla{f_{\text{cls}}\left(V \setminus S_{A} \right)} \cdot \alpha$ and $f_{\text{reg}} \left(V \setminus (S_{A} + \alpha) \right) = f_{\text{reg}} \left( V \setminus S_{A} \right) - \nabla{f_{\text{reg}}\left(V \setminus S_{A} \right)} \cdot \alpha$. The model output can be viewed as a negative gain process, i.e., $f_{\text{reg}} \left(V \setminus (S_{A} + \alpha) \right) = \boldsymbol{b} - \boldsymbol{b}^{\ast}$. Assuming that the searched $\alpha$ is valid, i.e., $\nabla{f_{\text{reg}}}>0$ and $\nabla{f_{\text{cls}}}>0$. We have:
\begin{equation}\label{collaboration_monotonic}
    \begin{aligned}
    & s_{\text{colla.}}(S_{A}+\alpha, \boldsymbol{b}_{\text{target},c}) - s_{\text{colla.}}(S_{A},\boldsymbol{b}_{\text{target}},c) \\
    =& \max_{\boldsymbol{b}_i \in f_{\text{reg}}(V \setminus S_{A}), s_{c,i}\in f_{\text{cls}}(V \setminus S_{A})}\text{IoU}(\boldsymbol{b}_{\text{target}}, \boldsymbol{b}_i)\cdot s_{c,i} \\
    &-\max_{\boldsymbol{b}_i \in f_{\text{reg}}(V \setminus (S_{A}+\alpha)), s_{c,i}\in f_{\text{cls}}(V \setminus (S_{A}+\alpha))}\text{IoU}(\boldsymbol{b}_{\text{target}}, \boldsymbol{b}_i)\cdot s_{c,i} \\
    =& 1- s_{\text{colla.}}(S_{A},\boldsymbol{b}_{\text{target}},c) \\
    & - \min ( 1- s_{\text{colla.}}(S_{A},\boldsymbol{b}_{\text{target}},c), \max_{\boldsymbol{b}_i \in \boldsymbol{b}^{\ast}, s_{c,i} \in \boldsymbol{s}^{\ast}_c} \text{IoU}(\boldsymbol{b}_{\text{target}}, \boldsymbol{b}_i)\cdot s_{c,i}) \\
    =&\max (0, \max_{\boldsymbol{b}_i \in \boldsymbol{b}^{\ast}, s_{c,i} \in \boldsymbol{s}^{\ast}_c} \nabla{f_{\text{reg}}\left( V \setminus S_{A} \right)} \cdot \nabla{f_{\text{cls}}\left( V \setminus S_{A} \right)} \cdot \alpha^{2})  \\
    \ge&0,
    \end{aligned}
\end{equation}
the collaboration score satisfies the monotonic non-negative property in the process of maximizing the marginal effect. Since $S_A \subseteq S_B \subseteq V$, more candidate boxes will be deleted, thus, $\boldsymbol{b}_{A}^{\ast} > \boldsymbol{b}_{B}^{\ast}$, and $\nabla{f_{\text{reg}}\left( V \setminus S_{A} \right)} > \nabla{f_{\text{reg}}\left( V \setminus S_{B} \right)}$. Since only a small number of candidate boxes $\boldsymbol{b}^{\ast}$ are removed, $\nabla{f_{\text{cls}}}(S_{A})\cdot\alpha$ can be regarded as a tiny constant and can be ignored.
Combining Eq.~\ref{collaboration_monotonic}, we have:
\begin{equation}\label{colla_submodular}
    \begin{aligned}
        &s_{\text{colla.}}(S_{A}+\alpha, \boldsymbol{b}_{\text{target},c}) - s_{\text{colla.}}(S_{A},\boldsymbol{b}_{\text{target}},c) > \\
        &s_{\text{colla.}}(S_{B}+\alpha, \boldsymbol{b}_{\text{target},c}) - s_{\text{colla.}}(S_{B},\boldsymbol{b}_{\text{target}},c).
    \end{aligned}
\end{equation}

We can prove that both the Clue Score and Collaboration Score satisfy submodularity under certain conditions. Since any linear combination of submodular functions is itself submodular~\cite{fujishige2005submodular}, we have:
\begin{equation}
    \mathcal{F}\left(S_{A} \cup \{ \alpha \}\right) - \mathcal{F}\left(S_{A} \right) \ge \mathcal{F}\left(S_{B} \cup \{ \alpha \}\right) - \mathcal{F}\left(S_{B} \right),
\end{equation}
and we can prove that Eq.~\ref{submodular_function} satisfies the submodular properties.
\end{proof}

From the above derivation, we find that the gain or negative gain condition of bounding boxes $\boldsymbol{b}^{\ast}$ is critical for satisfying submodularity. Thus, the detection model $f$, which can return relevant confidence candidate boxes for any combination of input sub-regions, is theoretically guaranteed to meet the required boundaries. Most detection models fulfill this condition by not filtering out low-confidence candidate boxes. However, multimodal large language model-based detection models, which may not directly output candidate boxes or confidence scores, do not fully satisfy these assumptions. This highlights room for improvement in explaining the results of such models.

\section{Faithfulness in Traditional Detectors}

We also validated the effectiveness of our interpretation method on traditional object detectors, including the two-stage detector Mask R-CNN (ResNet-50 backbone with Feature Pyramid Networks)~\cite{he2018mask} and the one-stage detectors YOLO v3 (DarkNet-53 backbone)~\cite{redmon2018yolov3}, FCOS (ResNet-50 backbone with Feature Pyramid Networks)~\cite{tian2020fcos}, and SSD~\cite{liu2016ssd}. We use the pre-trained models provided by MMDetection 3.3.0\footnote{\url{https://github.com/open-mmlab/mmdetection}} for interpretation. Following the evaluation settings of D-RISE~\cite{petsiuk2021black} and ODAM~\cite{zhao2024gradient_detector}, we selected samples for interpretation that were correctly predicted by the model with high confidence and precise localization.

Table~\ref{faithfulness_on_traditional} presents the results. We observe that, when considering both location and classification faithfulness metrics, D-RISE emphasizes location information more than ODAM, resulting in better performance in this aspect. In contrast, ODAM demonstrates higher faithfulness to classification scores on certain detectors, such as SSD and YOLO v3. Notably, ODAM outperforms D-RISE in the location metric Point Game. While ODAM and D-RISE have specific advantages across different metrics and models, our method consistently achieves state-of-the-art results across all models and metrics. On Mask R-CNN, our method outperforms D-RISE by 18.3\% in Insertion and 31.7\% in ODAM, as well as by 39.9\% and 37.9\% in Deletion. On YOLO V3, our method outperforms D-RISE by 15.3\% in Insertion and 18.4\% in ODAM and by 25.5\% and 49.1\% in Deletion. On FCOS, our method surpasses D-RISE by 30.0\% in Insertion and 34.7\% in ODAM, and by 27.4\% and 16.7\% in Deletion. On SSD, our method outperforms D-RISE by 41.2\% in Insertion and 17.7\% in ODAM, and by 21.6\% and 47.1\% in Deletion.

From the above results, we found that our method remains highly interpretable even on traditional object detection models, demonstrating its versatility in explaining both modern multimodal foundation models and traditional smaller detectors. Figure~\ref{traditional_detector_visualization} presents the visualization results, demonstrating that our method maintains high faithfulness in explaining traditional detectors.

\begin{figure}[!t]
    \centering
    \includegraphics[width=0.48\textwidth]{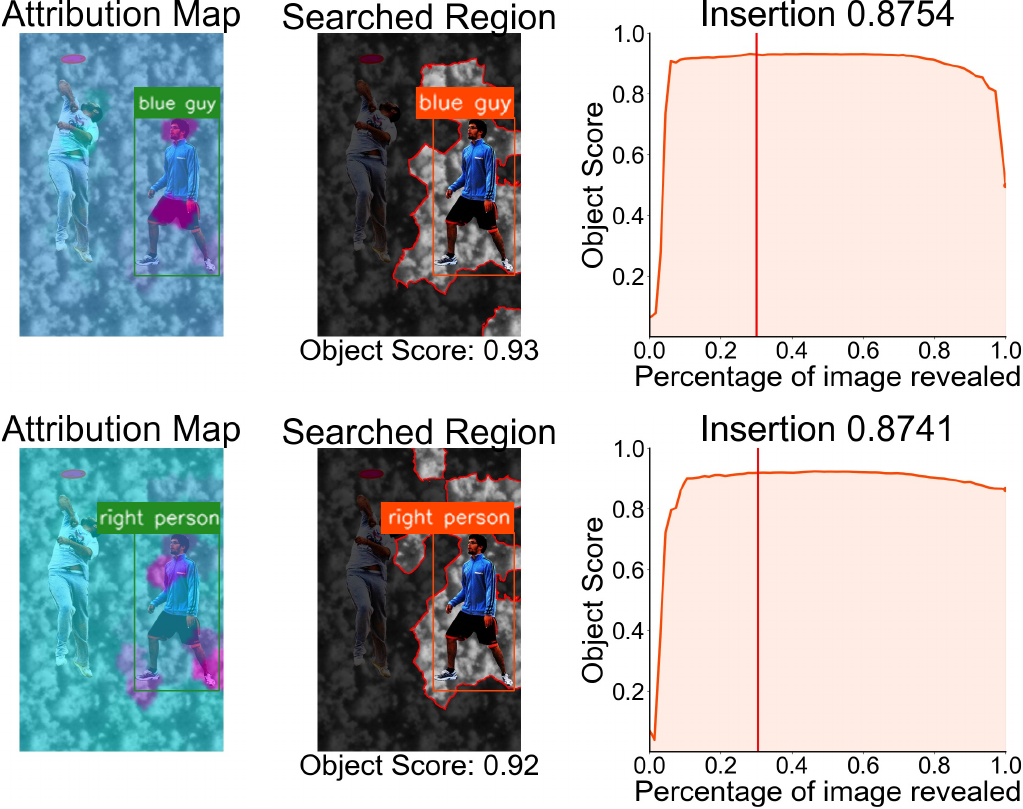}
    \caption{Interpretation of the same instance with different text expressions.} 
    \label{diverse_semantic} 
\end{figure}

\section{Semantic Interpretations}

We apply our interpretation method to the visual grounding task using Grounding DINO, focusing on explaining the same location corresponding to different text expressions. As shown in Figure~\ref{diverse_semantic}, although the important regions are similar when grounding the same object with different texts, the saliency map reveals distinct differences. For the interpretation of the text ‘blue guy’, the person wearing white clothes contributes less than the background region. In contrast, the interpretation of the ‘right person’ highlights only the correct object.

\section{Computational complexity}

Solving Eq.~\ref{problem_objective} is an $\mathcal{NP}$-hard problem, and the time complexity is $\mathcal{O}(2^{|V|})$. By employing a greedy search algorithm, we sort all subregions, resulting in a total number of inferences equal to $\frac{1}{2}|V|^2 + \frac{1}{2}|V|$, the algorithm’s time complexity is $\mathcal{O}(\frac{1}{2}|V|^2 + \frac{1}{2}|V|)$.

\section{Statistical Analysis}

We utilize Grounding DINO’s correct prediction attribution map and compute the numerical improvements of our method over the baseline for each sample. We then visualize the overall distribution of these improvements. As shown in Figure~\ref{distribution}, the distributions on the MS COCO, RefCOCO, and LVIS V1 datasets illustrate that our method outperforms the baseline on the majority of samples, demonstrating its superior performance.

\begin{figure}[!t]
    \centering
    \includegraphics[width=0.48\textwidth]{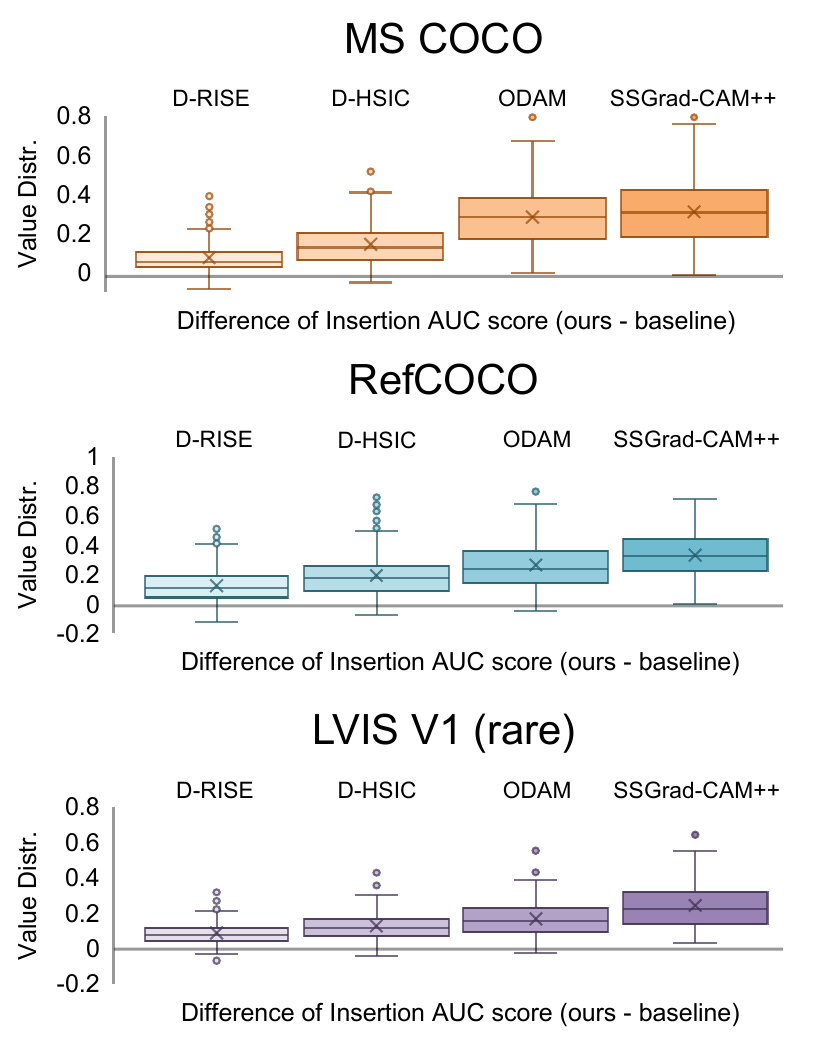}
    \caption{Distribution of our improvements over the baseline per sample on MS COCO, RefCOCO, and LVIS V1 datasets.
    } 
    \label{distribution}    
    \vspace{-16pt}
\end{figure}

\section{Evaluation Metrics}

In this paper, we adopt 6 faithfulness metrics. Given the object location box information, $\boldsymbol{b}_{\text{target}}$, and the target category, $c$, that requires explanation. In this section we will formulate a description of faithfulness metrics.


For the \textbf{Deletion AUC score}~\cite{petsiuk2021black}, which quantifies the reduction in the model’s ability of both location and classification when important regions are replaced with a baseline value. A sharp decline in performance indicates that the explanation method effectively identifies the key variables influencing the decision. Let $\mathbf{x}_{[x_{T}=x_{0}]}$ denote the input where the $T$ most important variables, according to the attribution map, are set to the baseline value $x_{0}=0$. Given a set $\mathcal{T} = \{T_0,T_1,\cdots, T_n\}$, where $T_0=0$ and $T_n$ is the input size of $\mathbf{x}$, this set represents the selected numbers of the most important regions. Then, the Deletion AUC score is given by:
\begin{equation}
    \begin{aligned}
        &\text{Del.} = \\
        &\sum_{i=1}^{n} \frac{\left (
        s_{\text{clue}}(\mathbf{x}_{[\mathbf{x}_{T_i}=x_{0}]})+s_{\text{clue}}(\mathbf{x}_{[\mathbf{x}_{T_{i-1}}=x_{0}]}) \right )\cdot \left ( T_{i}-T_{i-1}\right )}{2T_n},
    \end{aligned}
\end{equation}
the lower this metric, the better the attribution performance.

For the \textbf{Insertion AUC score}~\cite{petsiuk2021black}, which quantifies the increase in the model’s output as important regions are progressively revealed. This metric is defined as follows:
\begin{equation}
    \begin{aligned}
        &\text{Ins.} = \\
        &\sum_{i=1}^{n} \frac{\left (
        s_{\text{clue}}(\mathbf{x}_{[\mathbf{x}_{\bar{T}_i}=x_{0}]})+s_{\text{clue}}(\mathbf{x}_{[\mathbf{x}_{\bar{T}_{i-1}}=x_{0}]}) \right )\cdot \left ( T_{i}-T_{i-1}\right )}{2T_n},
    \end{aligned}
\end{equation}
where $\mathbf{x}_{[x_{\bar{T}}=x_0]}$ denotes the input where elements not belonging to the set $T$ are set to the baseline value $x_0=0$. The higher this metric, the better the attribution performance.

For the \textbf{Deletion AUC score (class)}, we first define $s_{cls}$, whose goal is to select the category score of the bounding box that is close to the explanation target:
\begin{equation}
    s_{cls}(S) = \arg\max_{s_{c,i}\in f(S)}{\text{IoU}(\boldsymbol{b}_{\text{target}},\boldsymbol{b}_i)\cdot s_{c,i}},
\end{equation}
then,
\begin{equation}
    \begin{aligned}
        &\text{Del. (class)} = \\
        &\sum_{i=1}^{n} \frac{\left (
        s_{cls}(\mathbf{x}_{[\mathbf{x}_{T_i}=x_{0}]})+s_{cls}(\mathbf{x}_{[\mathbf{x}_{T_{i-1}}=x_{0}]}) \right )\cdot \left ( T_{i}-T_{i-1}\right )}{2T_n}.
    \end{aligned}
\end{equation}

Similar, for the \textbf{Insertion AUC score (class)}, 
\begin{equation}
    \begin{aligned}
        &\text{Ins. (class)} = \\
        &\sum_{i=1}^{n} \frac{\left (
        s_{cls}(\mathbf{x}_{[\mathbf{x}_{\bar{T}_i}=x_{0}]})+s_{cls}(\mathbf{x}_{[\mathbf{x}_{\bar{T}_{i-1}}=x_{0}]}) \right )\cdot \left ( T_{i}-T_{i-1}\right )}{2T_n}.
    \end{aligned}
\end{equation}

For the \textbf{Deletion AUC score (IoU)}, we first define $s_{iou}$, whose goal is to select the IoU score of the bounding box that is close to the explanation target:
\begin{equation}
    s_{iou}(S) =\text{IoU}\left( \arg\max_{\boldsymbol{b}_i\in f(S)}{\text{IoU}(\boldsymbol{b}_{\text{target}},\boldsymbol{b}_i)\cdot s_{c,i}}, \boldsymbol{b}_{\text{target}}\right),
\end{equation}
then,
\begin{equation}
    \begin{aligned}
        &\text{Del. (IoU)} = \\
        &\sum_{i=1}^{n} \frac{\left (
        s_{iou}(\mathbf{x}_{[\mathbf{x}_{T_i}=x_{0}]})+s_{iou}(\mathbf{x}_{[\mathbf{x}_{T_{i-1}}=x_{0}]}) \right )\cdot \left ( T_{i}-T_{i-1}\right )}{2T_n}.
    \end{aligned}
\end{equation}

Similar, for the \textbf{Insertion AUC score (IoU)}, 
\begin{equation}
    \begin{aligned}
        &\text{Ins. (IoU)} = \\
        &\sum_{i=1}^{n} \frac{\left (
        s_{iou}(\mathbf{x}_{[\mathbf{x}_{\bar{T}_i}=x_{0}]})+s_{iou}(\mathbf{x}_{[\mathbf{x}_{\bar{T}_{i-1}}=x_{0}]}) \right )\cdot \left ( T_{i}-T_{i-1}\right )}{2T_n}.
    \end{aligned}
\end{equation}

\section{Limitation and Discussion}

\textbf{Limitations:} The main limitation of our method is that (i) \textbf{\textit{Sparse division}} impacts attribution faithfulness. When sub-regions mix positively and negatively contributing regions, attribution direction may be distorted. Refining sparse division strategies for different scenarios remains a region for improvement. (ii) A large number of sub-regions poses a challenge for \textbf{\textit{attribution time}}, as greedy search remains computationally demanding. Enhancing search efficiency or integrating external knowledge to filter unnecessary sub-regions can help accelerate attribution.


\textbf{Why object score decrease sometimes:} This phenomenon occurs because not all sub-regions contribute positively to the model’s decision, underscoring a key advantage of our method: maximizing the decision response with fewest sub-regions. Exposing additional sub-regions may lower the object score, revealing that the remaining regions negatively impact the decision. This effect is more pronounced in incorrect decisions, where certain regions may cause errors. If such regions are excluded, the decision could potentially be corrected.

\textbf{Future outlook:} Our method primarily focuses on interpretable attribution at the input level of the model. There remains significant potential for attributing internal parameters, particularly in transformer-based models. This approach could extend to explaining additional tasks, such as instance segmentation. Future research could explore improving models based on this mechanism by identifying and correcting problematic parameters.

\section{Actual Application} 

Attribution has numerous potential applications. Beyond aiding human understanding of model decisions, it can also help identify the causes of errors, enabling the analysis of potential hallucinations~\cite{chen2024less}. Some studies explore using attribution to guide model training and enhance performance~\cite{gao2024going}, while others investigate detecting anomaly decisions by assessing whether the attribution distribution deviates from expected patterns during deployment~\cite{stocco2022thirdeye,shu2024informer}. These diverse applications highlight the significant research value of attribution methods.

\section{Additional Ablation}

\textbf{Ablation on thresholding the confidence scores:} We discuss the impact of applying a confidence score threshold versus using all model predictions regardless of their confidence scores. As shown in Table~\ref{ablation_threshold}, applying a threshold to filter boxes leads to a consistent decline in all faithfulness metrics as the threshold increases. Therefore, we recommend avoiding the use of thresholds.

\begin{table*}[!t]
    \caption{Ablation on the confidence score threshold for Grounding DINO using the MS COCO validation set.}
    \label{ablation}
    \begin{center}
    \resizebox{0.8\textwidth}{!}{
        \begin{tabular}{c|ccccccc}
        \toprule
        \multirow{2}{*}{Threshold} & \multicolumn{7}{c}{Faithfulness Metrics}  \\ 
         & Ins. ($\uparrow$)  & Del. ($\downarrow$) & Ins. (class) ($\uparrow$)  & Del. (class) ($\downarrow$) & Ins. (IoU) ($\uparrow$)  & Del. (IoU) ($\downarrow$) & Ave. high. score ($\uparrow$) \\ \midrule
        None  & \textbf{0.5459} & \textbf{0.0375} & \textbf{0.6204} & \textbf{0.0882} & \textbf{0.8581} & \textbf{0.3300} & \textbf{0.6873} \\
        0.1  & 0.5267 & 0.0396 & 0.6007 & 0.0896 & 0.8498 & 0.3321 & 0.6660 \\
        0.2  & 0.5165 & 0.0423 & 0.5928 & 0.0916 & 0.8373 & 0.3408 & 0.6625 \\
        0.35  & 0.4862 & 0.0641 & 0.5638 & 0.1098 & 0.8023 & 0.3825 & 0.6519 \\  \bottomrule
        \end{tabular}
    }
    \end{center}
    \label{ablation_threshold}
\end{table*}

\textbf{Ablation on using confidence score}: We discuss the impact of whether or not to use the Confidence score on the interpretation. In Table~\ref{ablation_score}, without conf. score, both the attribution faithfulness for classes and the location will decrease, leading to imprecise attribution.

\begin{table*}[!t]
    \caption{Ablation on the confidence score for Grounding DINO using the MS COCO validation set.}
    \label{ablation}
    \begin{center}
    \resizebox{0.8\textwidth}{!}{
        \begin{tabular}{c|ccccccc}
        \toprule
        \multirow{2}{*}{Submodular function} & \multicolumn{7}{c}{Faithfulness Metrics}  \\ 
         & Ins. ($\uparrow$)  & Del. ($\downarrow$) & Ins. (class) ($\uparrow$)  & Del. (class)($\downarrow$) & Ins. (IoU) ($\uparrow$)  & Del. (IoU)($\downarrow$) & Ave. high. score ($\uparrow$) \\ \midrule
        w/ conf. score  & \textbf{0.5459} & \textbf{0.0375} & \textbf{0.6204} & \textbf{0.0882} & \textbf{0.8581} & \textbf{0.3300} & \textbf{0.6873} \\
        w/o conf. score  & 0.3725 & 0.0917 & 0.4410 & 0.1622 & 0.8051 & 0.3421 & 0.5928 \\ \bottomrule
        \end{tabular}
    }
    \end{center}
    \label{ablation_score}
\end{table*}

\section{More Visualization}

We present additional attribution visualizations for samples correctly predicted by Grounding DINO, including results from the MS COCO dataset to explain the object detection task (Figure~\ref{more_coco}) and from the RefCOCO dataset to illustrate the visual grounding task (Figure~\ref{more_refcoco}).

\begin{figure*}[!t]
    \centering
    \includegraphics[width=\textwidth]{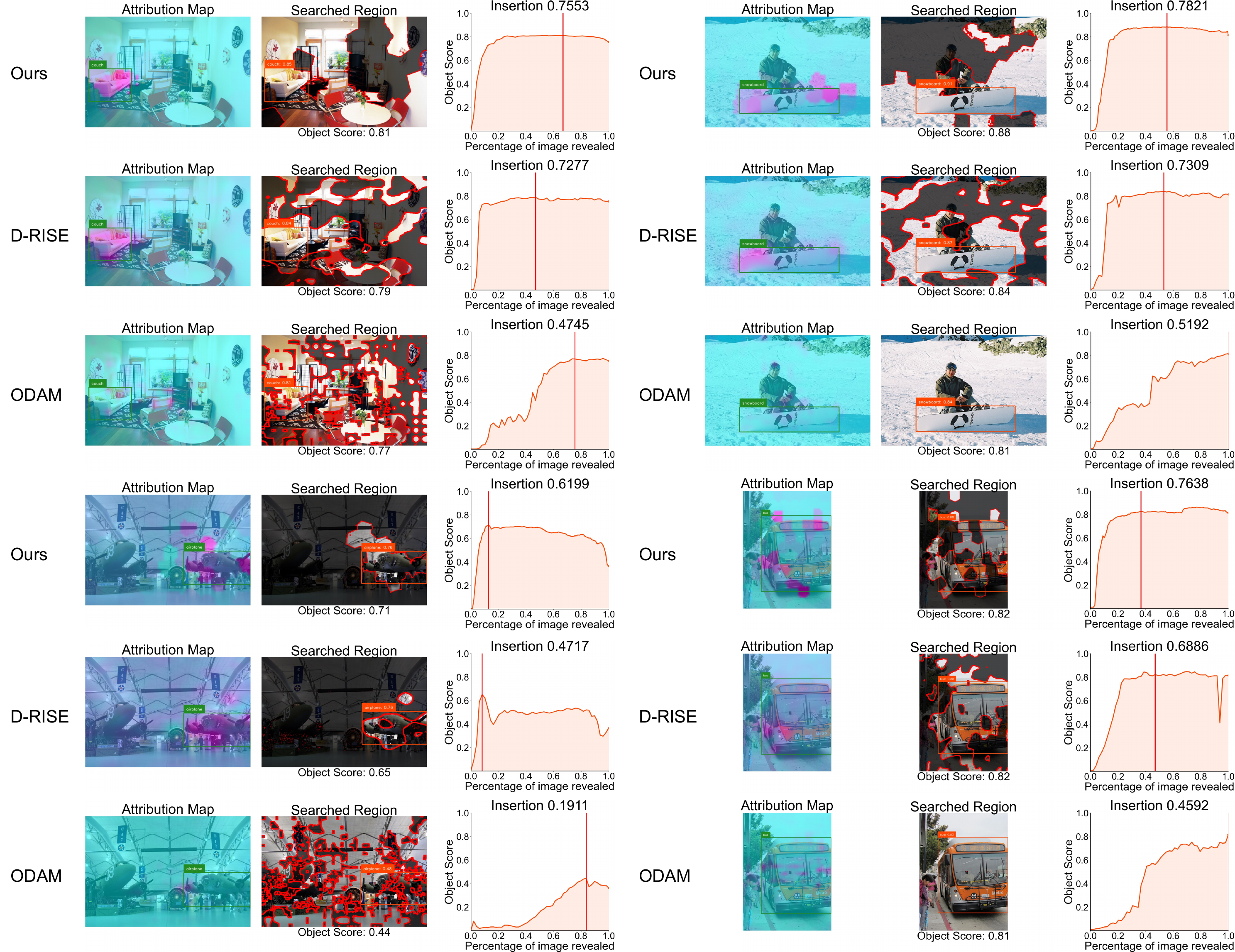}
    \caption{More visualization results of Grounding DINO for interpreting object detection task on the MS COCO dataset.} 
    \label{more_coco} 
\end{figure*}

\begin{figure*}[!t]
    \centering
    \includegraphics[width=\textwidth]{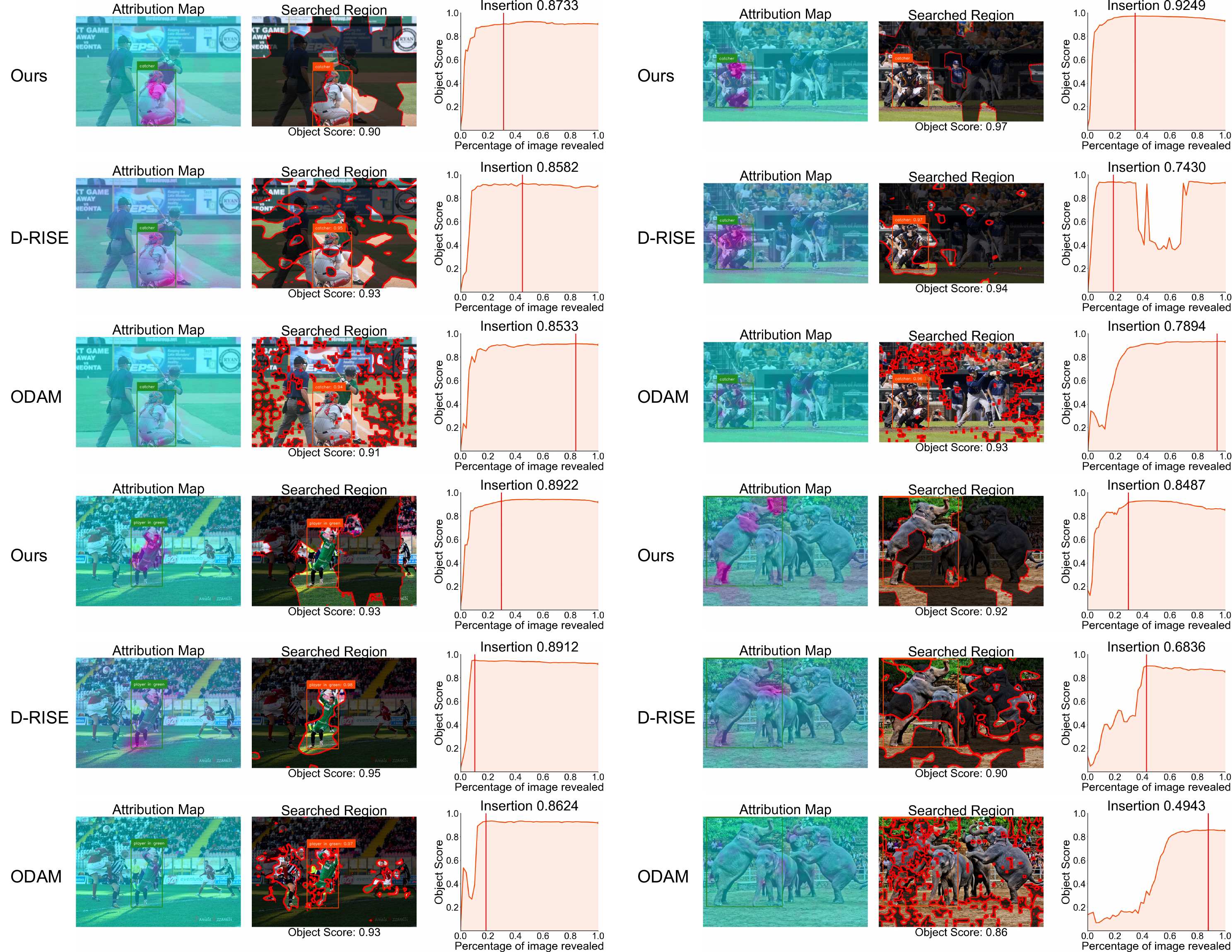}
    \caption{More visualization results of Grounding DINO for interpreting visual grounding task on the RefCOCO dataset.} 
    \label{more_refcoco} 
\end{figure*}

\end{document}